\documentclass[oneside,11pt]{article}

\usepackage[backend=biber, style=ieee, isbn=false]{biblatex}
\AtEveryBibitem{
    \clearfield{month}
}

\usepackage[symbol*]{footmisc}
\setfnsymbol{wiley}

\usepackage{amsmath}
\usepackage{graphicx}
\usepackage{enumitem}
\usepackage{url} 

\usepackage{microtype}
\usepackage{fancyhdr}
\usepackage{amssymb}
\usepackage{amsthm}
\usepackage{mathrsfs}
\usepackage[algo2e]{algorithm2e} 
\usepackage{booktabs}
\usepackage{longtable}
\usepackage{subcaption}
\makeatletter
\renewcommand{\@algocf@capt@plain}{above}
\makeatother

\newtheorem{theorem}{Theorem}
\newtheorem{lemma}{Lemma}
\newtheorem{proposition}[lemma]{Proposition}
\newtheorem{definition}{Definition}
\newtheorem{assumptionA}{}

\newtheorem{assumptionC}{}

\newcommand{\E}{\mathbb{E}}
\renewcommand{\P}{\mathbb{P}} 
\newcommand{\R}{\mathbb{R}} 
\newcommand{\N}{\mathbb{N}} 
\newcommand{\abs}[1]{\lvert #1\rvert} 
\newcommand{\cI}{\mathcal{I}} 
\DeclareMathOperator*{\argmax}{argmax} 
\newcommand{\1}{\mathbf{1}} 
\newcommand{\x}{\mathbf{x}}

\newcommand{\roo}{{\mathrm{root}}} 
\newcommand{\leaf}{{\mathrm{leaf}}} 
\newcommand{\parent}{\mathrm{F}}
\newcommand{\xtest}{\x^{\mathrm{test}}}
\newcommand{\ptest}{\cP^*}

\newcommand{\bestTheta}{\theta^{\mathrm{best}}} 

\newcommand{\mt}{m_{\mathrm{try}}} 
\newcommand{\Mt}{M_{\mathrm{try}}} 

\newcommand{\cP}{\mathcal{P}} 
\newcommand{\cF}{\mathcal{F}} 
\newcommand{\cO}{\mathcal{O}} 
\newcommand{\cD}{\mathcal{D}} 
\newcommand{\cT}{\mathcal{T}} 
\newcommand{\sR}{\mathscr{R}} 

\newcommand{\pto}{\overset{p}{\to}}

\DeclareMathOperator{\DWP}{DWP}
\DeclareMathOperator{\fDWP}{FDWP}
\DeclareMathOperator{\PP}{PP}
\DeclareMathOperator{\PII}{PIntI}
\DeclareMathOperator{\PFI}{PFI}

\newcommand{\threshDWP}{\eta_{\DWP}}
\newcommand{\threshPP}{\eta_{\PP}}

\newcommand{\sS}{\mathscr{S}} 
\newcommand{\sU}{\mathscr{U}} 
\newcommand{\sV}{\mathscr{V}} 

\begin{document}

\def\spacingset#1{\renewcommand{\baselinestretch}%
{#1}\small\normalsize} \spacingset{1}


\title{\textbf{Provable Recovery of Locally Important Signed Features and Interactions from Random Forest}
  }
  \author{Kata Vuk\thanks{Faculty of Informatics and Data Science, University of Regensburg, Germany. These authors contributed equally to this work.},\quad
    Nicolas Alexander Ihlo\footnotemark[1],\quad
    Merle Behr\thanks{Faculty of Informatics and Data Science, University of Regensburg, Germany (Email: {merle.behr@ur.de})} \\[3mm]
    Faculty of Informatics and Data Science\\
    University of Regensburg, Germany}
  \maketitle

\pagestyle{fancy}
\spacingset{1.5}
\bigskip
\begin{abstract}
Feature and Interaction Importance (FII) methods are essential in supervised learning for assessing the relevance of input variables and their interactions in complex prediction models. In many domains, such as personalized medicine, global scores summarizing overall feature importance are insufficient; instead, local interpretations for individual predictions are essential. Random Forests (RFs) are widely used in these settings, and existing interpretability methods typically exploit tree structures and split statistics to provide model-specific insights. However, theoretical understanding of local FII methods for RF remains limited, making it unclear how to interpret high importance scores for individual predictions.
We propose a novel, local, model-specific FII method that identifies frequent co-occurrences of features along decision paths, combining global patterns with those observed on paths specific to a given test point.
We prove that our method consistently recovers the true local signal features and their interactions under a Locally Spike Sparse (LSS) model and also identifies whether large or small feature values drive a prediction. We illustrate the usefulness of our method and theoretical results through simulation studies and a real-world data example.
\end{abstract}

\noindent%
\vfill

\newpage

\section{Introduction} \label{sec:intro}

In supervised machine learning feature importance scores are widely used to identify which input variables are most relevant to a prediction task. These scores provide insights into model behavior and enhance interpretability of complex algorithms. Beyond individual features, understanding interactions among variables is often crucial, for example in genetics, where phenotypic outcomes depend on interactions among genetic variants (cf., e.g., \cite{wan_megasnphunter_2009, yoshida_snpinterforest_2011}). In many applications, the sign of a feature---whether a large or small value drives a prediction---is also essential for interpretation.
Often, practitioners require explanations for specific predictions rather than global patterns. In personalized medicine or financial risk assessment, it is more relevant to understand why a model made a particular prediction for an individual than to summarize overall trends. Local Feature and Interaction Importance (FII) methods address this need by identifying influential features and their interactions for individual observations.
Random Forests (RFs) \cite{breiman2001} are among the most popular machine learning algorithms, particularly in settings where model-specific interpretations are needed. Their tree structure and split statistics naturally support interpretability. 

In this paper, we focus on local FII scores for RF at the individual prediction level, including the sign of each feature, to enable model-specific interpretation.
A major limitation of many interpretability methods 
is the lack of theoretical guarantees. This creates challenges in applications where the meaning of a ``true'' feature interaction is unclear. In this paper, we propose a new method that provides precise statistical guarantees for the consistent recovery of local signed interactions under a \emph{Locally Spiky Sparse} (LSS) model assumption (see, e.g., \cite{basu_iterative_2018, kumbier2018, behr_provable_2022}). The LSS model is a discontinuous nonlinear regression model motivated by the thresholding behavior commonly observed in biological processes. Our approach enables practitioners to know exactly which types of local interaction patterns can be provably recovered, improving interpretability and trust in real-world applications.

To recover signed feature interactions from RF ensembles, we build on the methodology introduced by the iRF algorithm \cite{basu_iterative_2018} and its signed variants in \cite{kumbier2018,wang2025}, and especially \cite{behr2024}. We identify sets of features that frequently co-occur along decision paths in the forest, assigning a sign to each feature based on the split direction at tree nodes. In contrast to previous work, here, we focus on methodology and theory for interpreting individual-level predictions for a specific test point. To this end, we combine global and local co-occurrence patterns of signed feature groups. Global prevalence aggregates frequencies across all paths, while local prevalence considers only paths traversed by the specific test observation. Thresholding both, global and local prevalence, yields our final local interaction method, LocalLSSFind, for which we prove consistency under the LSS model assumption.

\subsection{Related work}

Prominent model-agnostic local FII methods include LIME \cite{ribeiro2016should} and SHAP-based methods \cite{vstrumbelj2014explaining,lundberg2017unified}, 
with implementations for RFs discussed in \cite{lundberg_local_2020}. A recent model-specific feature importance approach is the local MDI+ method \cite{liang2025local}. 
For most FII methods proposed in the literature, theoretical and statistical understanding is very limited or entirely absent. In what follows, we review existing theoretical work on local FII methods.

For SHAP-based approaches---cf.~\cite{grabisch1999axiomatic,vstrumbelj2014explaining,lundberg2017unified}
---there is a solid theoretical foundation regarding the functional decomposition that individual FII scores correspond to, expressed as an expansion of the prediction model being explained. This precise decomposition is, in fact, the main motivation for SHAP values. Similarly, for LIME, some theoretical results provide insights into the functional approximation underlying the method \cite{garreau2020}.
However, these theoretical insights do not provide a statistical understanding of SHAP values---particularly their behavior with respect to signal and noise features in the data-generating process. See Section~\ref{sec:simulation} for numerical simulation examples that illustrate this point. In contrast, for LocalLSSFind, we establish a theoretical framework that characterizes its statistical properties and its ability to recover the true underlying signal features and interactions of the data-generating process under the LSS model.

There are statistical approaches for global FII methods---primarily for feature importance rather than feature interactions---that demonstrate certain methods yield zero scores in expectation for noise features. 
E.g., \cite{zhou2020} and \cite{loecher2020} provide such results for a modified MDI (mean decrease in impurity) score for RF, i.e., a model-specific method similar to LocalLSSFind. However, these results only address noise features and do not establish that signal features are consistently detected, in contrast to LocalLSSFind, which provides such guarantees for signed features and interactions and also covers the local, sample-specific case.
\cite{benard2022a} derives the asymptotic behavior of the MDA (mean decrease in accuracy) score for RFs in a general regression setting. However, these results also apply only to global scores and do not address signed interactions. Moreover, they rely on the assumption of a continuous regression function, which does not hold for the LSS model considered here.
\cite{vanderlaan2006} and \cite{williamson2023a} consider some global variable importance parameter defined for a general data-generating process and provide consistent nonparametric estimators. Similar other approaches are based on some form of conditional independence tests; see, e.g., \cite{watson2021a} and references therein. 
However, in contrast to LocalLSSFind, these approaches are model-agnostic and do not explore RF-specific behavior. They also do not operate at a local level and do not cover signed interaction effects, as LocalLSSFind does.

In summary, to the best of our knowledge, no other local, RF-specific FII method establishes consistent recovery of signal features and their signed interactions, as shown here with LocalLSSFind.

Statistical guarantees have also been developed for other interpretability approaches. For example, \cite{gan2023} propose model-agnostic confidence intervals for LOCO (leave-one-covariate-out) feature importance, though only for global scores; their results do not imply consistent recovery of signal features and do not address interactions. In a different direction, \cite{benard2021a} study the extraction of rule sets as interpretable prediction models rather than feature or signed interaction importance scores and establish consistency results for this framework. However, these methods address different interpretability objectives and do not provide guarantees for the local recovery of signed feature interactions.

\subsection{Paper structure}
The remainder of this paper is organized as follows.
Section 2 describes the LocalLSSFind methodology and explains how it summarizes the prevalence of joint feature appearances along decision paths in an RF tree ensemble.
Section 3 outlines the model assumptions regarding the underlying data-generating process, along with additional RF-related assumptions required for our main theoretical consistency results. These consistency results, concerning FII for LocalLSSFind, are presented in Section 4.
Section 5 reports simulation studies and a real-data application that demonstrate the practical performance of LocalLSSFind and highlight its advantages over existing approaches.
Section 6 concludes with a discussion. The appendix includes additional simulation results, software implementations, and all technical proofs.

\section{Methodology}

In this section, we introduce the \emph{LocalLSSFind} method, for recovery of signed feature interactions of individual predictions.
Throughout the following, we consider a given labeled training dataset $\cD=\{(\x_1,y_1), \dots, (\x_n,y_n)\}$ with features $\x_i =(x_{i1}, \dots , x_{ip}) \in \R^p$ and labels $y_i \in \R$, $i = 1,\ldots, n$. Here, we only consider the regression setting, but we stress that an extension for the classification setting is straight forward. Moreover, we fix some specific test data point $\xtest = (x_1^*, \dots, x_p^*) \in \R^p,$ for which local signed interactions from an RF prediction model should be derived.
For this, LocalLSSFind explores the individual decision paths traversed by $\xtest$ within an RF tree ensemble. 

\subsection{Review of RF}
We start with a quick review of the RF algorithm, see \cite{breiman2001}. RF consists of an ensemble of individual decision trees $T$, each mapping from the feature space $\R^p$ to the label space $\R$. Each tree is constructed on a bootstrapped or subsampled dataset $\cD^{(T)}$ of the original dataset $\cD$. Conditioned on the data $\cD$, each tree in the ensemble is constructed independently of the others, and the overall prediction function of RF corresponds to the average of the different tree-functions. Any node $t$ within a tree $T$ corresponds to some hyper-rectangle $R_t \subset \R^p$. 
A split at the node $t$ corresponds to a feature $k_t \in [p]$, using the notation $[p] = \{1,\ldots, p\}$, together with a threshold $\theta_t \in \R$, which divides the hyper-rectangle $R_t$ into two hyper-rectangles $R_{t, l}(k, \theta) = \{\x \in R_t : x_k \leq \theta\}$ and $R_{t, r}(k, \theta) = \{\x \in R_t : x_k > \theta\}$, corresponding to the left and right child nodes. 
Each tree $T$ is grown using a recursive procedure, denoted as the CART (Classification and Regression Trees) algorithm, see \cite{breiman1984}. 
For any hyper-rectangle $R$ define the impurity as the variance of the outcomes for samples from $\cD^{(T)}$ in $R$:
\[ I_n(R) =  \frac{1}{N_n(R)} \sum_{(\x_i,y_i) \in \cD^{(T)}: \x_i \in R} (y_i - \bar{y}_R)^2, \]
where \[N_n(R) = \abs{\{(\x_i,y_i) \in \cD^{(T)}: \x_i \in R\}}\] denotes the number of samples in $R$ and $\bar{y}_R = \frac{1}{N_n(R)} \sum_{\x_i \in R} y_i$ denotes the label-average of the samples in $R$.
At each node $t$ RF first selects a subset $\Mt \subset [p]$ uniformly at random. The size of this subset $\mt = \abs{\Mt}$ is the major tuning parameter of RF. Then, the optimal split $(k_t, \theta_t) \in [p] \times \R$ is determined by maximizing the impurity decrease
\begin{equation}\label{eq:impdecrease}
    \begin{aligned}
         \Delta_I^n(R_{t, l}(k, \theta), R_{t, r}(k, \theta)) := \frac{N_n(R_t)}{n} I_n(R_t) &- \frac{N_n(R_{t, l}(k, \theta))}{n}I_n(R_{t, l}(k, \theta)) \\ &- \frac{N_n(R_{t, r}(k, \theta))}{n}I_n(R_{t, r}(k, \theta)).
    \end{aligned}
\end{equation}
For the realized split along $k_t$ at $\theta_t$, we use the following shorter notation:
\[ \Delta_I^n(t) := \Delta_I^n(R_{t, l}(k_t, \theta_t), R_{t, r}(k_t, \theta_t)) \]
with
\[ (k_t, \theta_t) = \argmax_{k \in \Mt, \theta \in \R} \Delta_I^n(R_{t, l}(k, \theta), R_{t, r}(k, \theta)). \]
The procedure terminates at a node $t$ if it contains just a single observations $N_n(R_t) =1$ or when all responses are identical, i.e., $I_n(R_t) = 0$.

\subsection{LocalLSSFind}
The methodology, LocalLSSFind, which we propose to extract local signed interactions from RF, explores the set of signed features at individual decision paths in the tree ensemble.
Each path $\cP$ in a tree $T$ consists of a sequence of nodes $t \in \{ 1,\dots,d, t_{\leaf}\}$, where $d$ represents the depth of the path and $t_{\leaf}$ is a leaf node. Along this path a sequence of signed features $(k_{1},b_{1}),\dots,(k_{d},b_{d})$ is associated, where $k_{t} \in [p]$ indicates the feature index and $b_{t}\in \{-1,+1 \}$ indicates the direction of the split for that feature at node $t.$ Here, $b_{t} = -1$ denotes a split that follows the $\leq$ direction, while $b_{t} = +1$ denotes a split that follows the $>$ direction. 
Analogously as in \cite{behr_provable_2022}, for a given RF tree ensemble depending on data $\cD$, the path $\cP$ of the tree $T$, and any fixed constant $\epsilon > 0$, we define $\hat{\cF}_{\epsilon}(\cP, T,\cD)$ to be the set of signed features on $\cP$ where the corresponding node in the RF had an impurity decrease of at least $\epsilon$, that is,
\begin{equation}\label{def:hat_f}
\begin{aligned}
\hat{\cF}_{\epsilon}(\cP, T,\cD)
:= \{ (k_t, b_t) :\; & t \text{ is an inner node of } \cP \text{ with } \Delta_I^n(t) \ge \epsilon, \\
& \text{and feature } k_t \text{ appears for the first time on } \cP \}.
\end{aligned}
\end{equation}
Next, we define the prevalence summary statistics of the RF tree ensemble that LocalLSSFind uses to extract signed interactions.
Conditioning on data $\cD$, let $T$ be a random tree grown in the RF with parameter $\mt$, and let $\cP$ denote a path of $T$ with depth $d$, selected randomly with probability $2^{-d}$. Note that randomly selecting a path $\cP$ in a tree $T$ is equivalent to starting at the root node of $T$, and at each subsequent node, choosing to go left or right with probability $50\,\%$.
Moreover, let $\ptest$ be the unique path of $T$ into which the test point $\xtest \in \R^p$ falls.
Let $\epsilon > 0$. 
For any signed feature set $S^\pm \subset [p] \times \{-1,+1\},$ the depth-weighted prevalence, $\DWP$, of $S^\pm$ is defined as the probability that $S^\pm$ appears on the random path $\cP$ within the set $\hat{\cF}_{\epsilon}$ (see \cite{behr_provable_2022}, Definition~3), i.e.,
\begin{equation}\label{eq:DWP}
    \DWP_\epsilon(S^\pm) := \P_{\cP,T}(S^\pm \subseteq \hat{\cF_\epsilon}(\cP, T, \cD) \mid \cD).
\end{equation}
Moreover, the $\xtest$-based path prevalence of $S^{\pm}$ is defined as the probability that $S^\pm$ appears on the path $\ptest$ for the random tree $T$ within the set $\hat{\cF}_{\epsilon}$, that is, 
\begin{equation}\label{eq:path-prevalence}
    \PP^*_{\epsilon}(S^{\pm}) := \P_T(S^{\pm} \subseteq \hat{\cF}_{\epsilon}(\ptest, T,\cD) \mid \cD).
\end{equation}
Note that, conditioned on the data $\cD$, one can generate as many random trees $T$ from the RF algorithm as desired. Hence, both $\DWP_\epsilon(S^\pm)$ and $\PP^*_{\epsilon}(S^{\pm})$ can be estimated with arbitrary accuracy from an RF with sufficiently many trees. Intuitively, $\DWP_\epsilon(S^\pm)$ captures how likely it is to observe $S^{\pm}$ globally, on any randomly selected path $\cP$ in the tree ensemble, and  $\PP^*_{\epsilon}(S^{\pm})$ captures how likely it is to observe $S^{\pm}$ on the specific paths in the ensemble where $\xtest$ falls into, restricted to nodes with an impurity decrease of at least $\epsilon$.
LocalLSSFind selects all signed interactions $S^{\pm} \in [p] \times \{-1,1\}$ for which both $\DWP_{\epsilon}(S^{\pm})$ and $\PP^*_{\epsilon}(S^{\pm})$ exceed a specified threshold, as detailed in Algorithm~\ref{Algo:LocalLSSFind}. The construction of the set $\sS_G$ in Algorithm~\ref{Algo:LocalLSSFind} is identical to the LSSFind algorithm of \cite{behr_provable_2022}. It returns all signed interactions $S^{\pm}$ whose depth-weighted prevalence exceeds a threshold, using a rescaling factor of $2^{\abs{S^{\pm}}}$ to make interactions of different sizes comparable, and retaining only minimal interactions and not proper supersets.
LocalLSSFind then applies an additional filtering step: only interactions whose local prevalence for the specific test point $\xtest$ also exceed a threshold are kept.

\begin{algorithm2e}
 \caption{LocalLSSFind($\cD$, $\mt$, $\epsilon$, $\threshDWP$, $\threshPP$, $s_{\max}$, $\xtest$)}\label{Algo:LocalLSSFind}
\SetKwInOut{Input}{Input}
\Input{Dataset $\cD$, RF hyperparameter $\mt$, impurity decrease threshold $\epsilon > 0$, prevalence thresholds $\threshDWP, \threshPP > 0$, maximum interaction size $s_{\max} \in \N$, and test data point $\xtest$.}
\SetKwInOut{Output}{Output}
\Output{A collection $\sS_{L}$ of sets of signed features.}
Train an RF using dataset $\cD$ with parameter $\mt$;\
$\tilde{\sS}_G := \{S^\pm \subset [p]\times \{-1, +1\} \text{ s.t.\ } \abs{S^\pm} \leq s_{\max} \text{ and } 2^{\abs{S^\pm}} \cdot \DWP_{\epsilon}(S^\pm) \geq 1 - \threshDWP \}$;\
$\sS_G := \{ S \in \tilde{\sS}_G \text{ s.t.\ there is no set } S' \in \tilde{\sS}_G \text{ with } S' \subsetneq S \}$;\
return $\{S^\pm \in \sS_G \text{ s.t.\ } \PP^*_{\epsilon}(S^\pm) \geq 1 - \threshPP \}$.
\end{algorithm2e}

\section{LSS model and assumptions}\label{sec:defs}

In this section, we introduce the data-generating model and provide details on the assumptions under which we prove the main consistency result for LocalLSSFind. 

\subsection{LSS model and local signed interactions}

We consider data generated from an LSS model \cite{basu_iterative_2018, kumbier2018, behr_provable_2022}, which assumes that the underlying regression function is a linear combination of Boolean interaction terms. Each term captures thresholded or discontinuous interactions among groups of features, a behavior commonly observed in biological processes. This modeling assumption provides a precise definition of true signed interactions in the data-generating process, based on feature groups and their signs in the threshold relationships within individual Boolean terms.
We stress that without specific modeling assumptions, the notion of an \emph{interaction of features} is ill-defined. Usually, an interaction is described as a deviation from additivity, but this depends on the function’s scale; for example, a multiplicative function becomes additive on a logarithmic scale. In fact, any multivariate real-valued function with compact support can be expressed as additive under an appropriate transformation \cite{girosi1989}. In contrast, the LSS model offers a rigorous mathematical definition of signed feature interactions driving the data-generating process. Moreover, it naturally specifies which interactions are \emph{locally} relevant---i.e., for a given test point---via the Boolean terms that are true (non-zero) at that prediction. Thus, the LSS model is not only well motivated by applications, but also provides a foundation for proving statistical consistency of local signed interaction recovery.
\begin{definition}[LSS model]\label{def:LSS_model}
    Consider labeled data $\cD=\{(\x_1,y_1), \dots, (\x_n,y_n)\}$ with $\x_i =(x_{i1}, \dots , x_{ip}) \in \R^p$ and labels $y_i \in \R$. Assume that the samples are i.i.d.\ from a distribution $\P(X,Y)$ with $X=(X_1, \dotsc, X_p)$, such that the regression function takes the following form:
    \begin{equation} 
        \E(Y \mid X) = \beta_0 + \sum_{j=1}^J \beta_j \prod_{k \in S_j } \1(X_k \gtreqless \gamma_k), \label{eq:E_lssmodel}
    \end{equation}
    where $\gtreqless$ means either $\leq$ or $\geq$, potentially different for every $k$. We assume that there exist fixed constants $C_{\beta}>0,~ C_{\gamma}\in (0,0.5),$ such that for the coefficients $\beta_j$ it holds that
    $ \min_{1\leq j \leq J} \abs{\beta_j} > C_{\beta},$
    and for the thresholds $\gamma_k, ~ k\in S_j, j=1, \dots, J,~ \gamma_k \in (C_{\gamma}, 1-C_{\gamma}).$
    $S_1, \ldots, S_J \subset [p]$ are sets of features called Basic Interactions (BIs). We associate $\leq$ in~\eqref{eq:E_lssmodel} with a negative sign $(b_k=-1)$ and $\geq$ with a positive sign $(b_k=+1)$, such that a \emph{signed feature} can be written as a tuple $(k, b_k) \in \{1, \dots, p\} \times \{-1, +1\}$. We call $S_1^\pm, \ldots, S_J^\pm\subset [p]\times \{-1,+1\}$ Basic Signed Interactions (BSIs) with $S_j^\pm = \{(k, b_k):k\in S_j\}$. For BIs with only one feature $k$, due to the sign ambiguity in the LSS model, i.e., $1(X_k \leq \gamma_k) = 1 - 1(X_k > \gamma_k)$, both ${(k , -1)}$ and ${(k , +1)}$ are considered as BSIs.
\end{definition}
To simplify notation, we assume without loss of generality that all inequalities in the LSS model are of the form $\leq$, i.e.,
\begin{equation}\label{eq:LSS_leq}
    \E(Y \mid X) = \beta_0 + \sum_{j=1}^J \beta_j \prod_{k \in S_j} \1(X_k \leq \gamma_k).
\end{equation} 
Thus, we define $S_j^- = \{(k,-1):k\in S_j\}$ and $S_j^+ = \{(k,+1): k\in S_j\}$. Note that for a BI involving only one feature $k$, the set $\{(k,+1)\}$ is still considered a BSI. Therefore, the BSIs in the LSS model in~\eqref{eq:LSS_leq} are given by $S_j^-$ for all $j$ together with $S_j^+$ where $\abs{S_j}=1$. 

In this paper, our focus lies on BSIs that are not only in the underlying LSS model but are especially relevant to a specific prediction of a new observation, i.e., a new test point $\xtest$. We define the basic (signed) interactions for $\xtest$, as follows.
\begin{definition}[Basic Interaction (BI) and Basic Signed Interaction (BSI) for $\xtest$] \label{def:BSI_xtest}
    Let $S_j^-$ be a BSI in the LSS model \eqref{eq:LSS_leq}. We define $S_j^{*-} = S_j^-$ to be a BSI for $\xtest$, if 
    \begin{equation}
        \prod_{k \in S_j^{*-} } \1(x^*_k \leq \gamma_k) = 1. \label{eq:BSI_xtest}
    \end{equation}
    Additionally, single-feature BSIs in the LSS model with positive sign $S_j^+ = \{(k, +1)\}$, where $x^*_k > \gamma_k$, are also defined to be BSIs for $\xtest$, i.e., $S_j^+=S_j^{*+}$. We denote the corresponding (unsigned) BI for $\xtest$ with $S_j^*$.
\end{definition}
Note that not every BI for the LSS model is also a BI for $\xtest$. Furthermore, note that since we assume that all inequalities in the LSS model are $\leq$, it is necessary and sufficient that $x^*_k\leq \gamma_k$ for all $k\in S^{-}_j$, in order for the BSI $S_j^-$ with $\abs{S_j^-} > 1$ to be a BSI for $\xtest$, i.e., $S_j^-=S_j^{*-}$. For a BI $S_j$ in the LSS model with $\abs{S_j} = 1$, both $S_j^-$ and $S_j^+$ are BSIs in the LSS model. However, if $x^*_k\leq \gamma_k$, then only $S_j^-$ is a BSI for $\xtest$, and if $x^*_k > \gamma_k$, then only $S_j^+$ is a BSI for $\xtest$.

\subsection{Model assumptions}
For our theoretical results, we require regularity constraints on the data generating process $\P(X,Y)$---such as independence between features, bounded response, and disjoint interaction sets to ensure identifiability---as also considered in \cite{behr_provable_2022}.
\begin{assumptionC}[Uniformity]\label{C:uniformity}
    The feature vector $X$ is uniformly distributed on $[0,1]^p$.
\end{assumptionC}
\begin{assumptionC}[Bounded-response]\label{C:bounded-response}
    The response variable $Y$ is bounded, w.l.o.g.\ we assume $\abs{Y} < 1$. 
\end{assumptionC}
\begin{assumptionC}[Non-overlapping basic interactions]
    The feature sets corresponding to different interactions, $S_1,\ldots, S_J$, do not overlap. Formally,
    $ S_{j_1} \cap S_{j_2} = \emptyset$ for all $j_1\neq j_2$.
\end{assumptionC}
\begin{assumptionC}[Sparsity]\label{C:sparsity}
    The number of signal features $s =\abs{\cup_{k = 1}^J S_j}$ is bounded, independent of $n$. The number of noise features can grow with $n$, such that $\frac{\log(p)}{n} \to 0$, as $n\to\infty$. 
\end{assumptionC}

\subsection{Assumptions on the RF tree ensemble}
Let $\mu(R)$ denote the volume of any hyper-rectangle $R$. 
We make the following assumptions on an RF tree ensemble (cf.\ assumptions A1--A4 in \cite{behr_provable_2022}):
\begin{assumptionA}[Increasing depth of a tree in the RF ensemble]\label{A:increasing_depth}
The minimum depth of any path in any tree goes to infinity as the sample size increases, i.e., \[\min_{T}\min_{t_\leaf\in T} D(t_\leaf)\pto \infty, \quad \text{as } n \to \infty.\]
\end{assumptionA}
\begin{assumptionA}[Balanced split in a tree of the RF ensemble]\label{A:balancedsplit}
Each split $(k_t, \theta_t)$ is balanced: for any node $t$,
\[ \min\left(\frac{\mu(R_{t,l}(k_t, \theta_t))}{\mu(R_{t,r}(k_t, \theta_t))},\frac{\mu(R_{t,r}(k_t, \theta_t))}{\mu(R_{t,l}(k_t, \theta_t))}\right) > \frac{C_\gamma}{1-C_\gamma}. \]
\end{assumptionA}
Note that, without loss of generality, we use the same $C_\gamma$ here as in the LSS model. Otherwise, we can always let $C_\gamma$ to be the minimum of the two.
\begin{assumptionA}[$\mt$ is of order $p$]\label{A:mtry}
$C_m p + (1 - C_m)s \leq \mt \leq (1 - C_m)(p -  s )$ where $C_m \in (0,0.5)$ is a constant.
\end{assumptionA}
\begin{assumptionA}[No bootstrap or subsampling of samples]\label{A:no_bootstrap}
All trees in the RF are grown on the entire dataset without bootstrapping or subsampling, i.e., $\cD^{(T)} = \cD$ for any $T$.
\end{assumptionA}
\ref{A:increasing_depth} is a reasonable assumption since we consider trees grown to full depth (as in typical RF implementations), where tree depth scales as $\cO(\log(n))$. \ref{A:balancedsplit} is a standard assumption for RF theory and can be easily incorporated into any classical implementation. \ref{A:mtry} requires that the $\mt$ parameter scales as $C \cdot p$ for some constant $C$, an assumption also used in other RF consistency proofs (see, e.g., \cite{klusowski2024}). \ref{A:no_bootstrap} is a technical assumption that simplifies notation and analysis. While subsampling is essential for other consistency results (cf. \cite{biau2012a, wager2018}), it is not needed here, as we focus solely on FII within the ensemble. For further discussion of these assumptions, see \cite{behr_provable_2022}.

\section{Main theoretical results}

Under the assumptions outlined in Section~\ref{sec:defs}, we can now state our main theoretical consistency results.
\begin{theorem}[Consistency of signed interaction importance]\label{thm:PP_finds_BSI}
   Suppose that the data $\cD$ is generated from the LSS model in Definition~\ref{def:LSS_model} with constraints \ref{C:uniformity}--\ref{C:sparsity}. Fix some test point $\xtest \in [0,1]^p$ independent of $\cD$ such that for all $k \in \cup_{j = 1}^J S_j$ we have $x^*_k \neq \gamma_k$. 
   
   Let $\sS_L$ denote the output of LocalLSSFind (Algorithm~\ref{Algo:LocalLSSFind}), where
   \begin{equation*}
       2^s\cdot b(\epsilon) < \threshDWP < \frac{C_m^s}{2} \quad \text{and} \quad b(\epsilon) < \threshPP < 1
   \end{equation*}
   with
   \begin{equation}\label{eq:bepsilon}
       b(\epsilon) = \left(\frac{ 4 \epsilon}{C_\beta^2 C_\gamma^{2\max_j \abs{S_j} -1}}\right)^{\tilde{C}},
   \end{equation}
   $\tilde{C}=C^{2s}_m/\log(1/C_{\gamma})$, and $s = \abs{\bigcup_j S_j}$.
   Assume that the trees in the RF are CART trees that satisfy assumptions \ref{A:increasing_depth}--\ref{A:no_bootstrap}.
   Then, for any fixed $\epsilon > 0$, with probability approaching one as $n\to\infty$, $\sS_L$ equals the set of basic signed interactions of $\xtest$ of size at most $s_{\max}$.
\end{theorem}
Note that $b(\epsilon) \to 0$ as $\epsilon \to 0$ in Theorem~\ref{thm:PP_finds_BSI}. Hence, Theorem~\ref{thm:PP_finds_BSI} guarantees that when the thresholds $\epsilon, \threshDWP, \threshPP$ are chosen small enough, LocalLSSFind indeed consistently recovers the true underlying signed interactions of the test point $\xtest$.

The proof of Theorem~\ref{thm:PP_finds_BSI} builds on two propositions about the local path prevalence $\PP^*_{\epsilon}(S^{\pm})$ which are used for the additional local filtering step of LocalLSSFind.
The first proposition shows that, asymptotically, for any RF trained on data from the LSS model, the path prevalence of a BSI for $\xtest$ (i.e., $\PP^*_{\epsilon}(S^{*\pm})$) is lower bounded by a quantity close to $1$ (Proposition~{\ref{theo:BSIlowerBound}}). In contrast, if $S^{\pm}$ is a BSI in the LSS model but not a BSI for the specific test point (i.e., $S^{\pm} \neq S^{*\pm}$), then its path prevalence converges in probability to zero (cf.\ Proposition~\ref{theo:nonBSIto0}).  

\begin{proposition}\label{theo:BSIlowerBound}
Let $T$ be a CART tree satisfying assumptions \ref{A:increasing_depth}--\ref{A:no_bootstrap} and suppose that the constraints \ref{C:uniformity}--\ref{C:sparsity} hold. Suppose that the data $\cD$ is generated from the LSS model and let $S^{*\pm}$ be a BSI for the test point $\xtest$. Then, for any fixed constant $\epsilon>0$, 
\[ \PP^*_{\epsilon}(S^{*\pm}) \geq 1 - b(\epsilon) + r_n(\cD, \epsilon), \]
where $r_n(\cD, \epsilon)  \pto 0$ as $n\to \infty$, and $b(\epsilon)$ as in~\eqref{eq:bepsilon}.
\end{proposition}

\begin{proposition}\label{theo:nonBSIto0}
 Let $T$ be a CART tree satisfying assumptions \ref{A:balancedsplit} and \ref{A:no_bootstrap} and suppose that the constraints \ref{C:uniformity}--\ref{C:sparsity} hold. Let $S^\pm$ be a BSI in the LSS model but not a BSI for the test point $\xtest$. Then, for any fixed $\epsilon > 0$
 \[ \PP^*_{\epsilon}(S^\pm) \pto 0 \quad \text{as } n \to \infty. \]
\end{proposition}
The proofs of Proposition~\ref{theo:BSIlowerBound} and Proposition~\ref{theo:nonBSIto0} are deferred to the appendix. 

\begin{proof}[Proof of Theorem~\ref{thm:PP_finds_BSI}]
    Define
    \[ \sV := \{S^\pm \subset [p] \times \{-1, 1\} \text{ s.t.\ } \abs{S^\pm} \leq s_{\max} \text{ and } \PP_\epsilon^*(S^\pm) \geq 1 - \threshPP\}. \]

    Let $\sU := \sS_G$ be the set as in Algorithm~\ref{Algo:LocalLSSFind}.
    By Theorem~3 of \cite{behr_provable_2022}\footnote{Note that the definition of $\sU$ in Theorem~3 of \cite{behr_provable_2022} uses $S \subsetneq S'$, which is a typo and should be $S' \subsetneq S$, as in Algorithm~\ref{Algo:LocalLSSFind}.}, $\sU$ is equal to the set of BSIs in the LSS model with size at most $s_{\max}$, with probability approaching one as $n\to\infty$. Since every BSI for $\xtest$ is also a BSI in the LSS model, it follows that 
    \[ \P_{\cD}(\sU \supseteq \{\text{BSIs for } \xtest \text{ of size at most } s_{\max}\}) \to 1. \]
    
    If $S^{*\pm}$ is a BSI for $\xtest$, then it is also a BSI in the LSS model. By Proposition~\ref{theo:BSIlowerBound}, \[\PP^*_{\epsilon}(S^{*\pm}) \geq 1 - b(\epsilon) + r_n(\cD, \epsilon),\] where $r_n(\cD, \epsilon) \pto 0$ as $n \to \infty$. 
    Since $b(\epsilon) < \threshPP$, we obtain
    \begin{align*}
        \P_{\cD}(1 - b(\epsilon) + r_n(\cD, \epsilon) \geq 1 - \threshPP) &= \P_{\cD}(r_n(\cD, \epsilon) \geq \underbrace{b(\epsilon) - \threshPP}_{<0}) \\
        &\geq \P_{\cD}(\abs{r_n(\cD, \epsilon) - 0} \leq \threshPP - b(\epsilon)) \to 1.
    \end{align*}
    Thus, with probability approaching 1 as $n \to \infty$, \[\PP^*_{\epsilon}(S^{*\pm}) \geq 1 - \threshPP.\] Therefore, if $S^{*\pm}$ has size at most $s_{\max}$, the probability that $\sV$ contains $S^{*\pm}$ also approaches 1. This holds for all BSIs for $\xtest$ of size at most $s_{\max}$, and since the number of such BSIs is bounded (constraint \ref{C:sparsity}), we can conclude that for $\sS_L = \sS_G \cap \sV$
    \begin{align*}
        &\P_{\cD}(\sS_L \supseteq \{\text{BSIs for } \xtest \text{ of size at most } s_{\max}\}) \\
        &= \P_{\cD}(\sS_L \cap \sV \supseteq \{\text{BSIs for } \xtest \text{ of size at most } s_{\max}\}) \to 1.
    \end{align*}
    
    If $S^\pm$ is not a BSI for $\xtest$ but is a BSI in the LSS model, then by Proposition~\ref{theo:nonBSIto0} we have $\PP^*_{\epsilon}(S^{\pm}) \pto 0$ for $n \to \infty$. This implies that the probability of $\PP^*_{\epsilon}(S^{\pm}) \geq 1-\threshPP$, and therefore the probability that $\sV$ contains $S^\pm$, approaches 0. Because this holds for all BSIs in the LSS model of size at most $s_{\max}$ that are not BSIs for $\xtest$, and because there are only finitely many such BSIs (by constraint \ref{C:sparsity}), we have
    \[ \P_{\cD}((\sV \cap \{\text{BSIs in LSS model}\})\setminus \{\text{BSIs for }\xtest\} = \emptyset) \to 1. \]
    
    Combining the above results and noting that $\sV$ contains only signed interactions of size at most $s_{\max}$, and \[\P_{\cD}(\sS_G = \{\text{BSIs in LSS model with size at most } s_{\max}\})\pto 1,\] 
    we conclude
    \begin{align*}
        &\P_{\cD}(\sS_L \subseteq \{\text{BSIs for } \xtest \text{ of size at most } s_{\max}\}) =\\
        & \P_{\cD}((\sS_G \cap \sV) \setminus \{\text{BSIs for } \xtest \text{ of size at most } s_{\max}\} = \emptyset) \geq \\
        & \P_{\cD}( \{(\sV \cap \{\text{BSIs in LSS model of size} \leq s_{\max}\})\setminus \{\text{BSIs for }\xtest \text{ of size} \leq s_{\max}\} = \emptyset \} \\
        &\phantom{=\P_{\cD}(}\cap \; \{\sS_G = \{\text{BSIs in LSS model with size at most } s_{\max}\} \})  \geq \\
        & \P_{\cD}((\sV \cap \{\text{BSIs in LSS model of size} \leq s_{\max}\})\setminus \{\text{BSIs for }\xtest \text{ of size} \leq s_{\max}\} = \emptyset) \\
        &\phantom{=}+\P_{\cD}(\sS_G = \{\text{BSIs in LSS model with size at most } s_{\max}\}) -1 \to 1
    \end{align*}
    as $n \to \infty$.
    Therefore, for the output $\sS_L=\sS_G \cap \sV$ of Algorithm~\ref{Algo:LocalLSSFind},
    \[ \P_{\cD}(\sS_L = \{\text{BSIs for } \xtest \text{ of size at most } s_{\max}\}) \to 1 \quad \text{as } n \to \infty. \]
\end{proof}

Clearly, any method which consistently recovers the set of BSIs for a test point $\xtest$ can also be used to consistently recover the individual signed features that drive the prediction of $\xtest$, simply by ignoring the interaction information. However, not all steps of LocalLSSFind are required to obtain such a consistency result for signed feature importance scores. In the following, we show that a simplified variant of LocalLSSFind suffices. Specifically, we consider a modified version of LocalLSSFind to recover local signed feature importance scores only, as detailed in Algorithm~\ref{Algo:LocalLSSFindFI}. Theorem~\ref{thm:consistencyLocalFI} establishes that the set of signed features produced by Algorithm~\ref{Algo:LocalLSSFindFI} asymptotically recovers exactly the signed signal features of the test point $\xtest$.

\begin{algorithm2e}
 \caption{LocalFeatureLSSFind($\cD$, $\mt$, $\epsilon$, $\threshDWP$, $\threshPP$, $s_{\max}$, $\xtest$)}\label{Algo:LocalLSSFindFI}
\SetKwInOut{Input}{Input}
\Input{Dataset $\cD$, RF hyperparameter $\mt$, impurity decrease threshold $\epsilon > 0$, prevalence thresholds $\threshDWP, \threshPP > 0$, maximum interaction size $s_{\max} \in \N$, and test data point $\xtest$.}
\SetKwInOut{Output}{Output}
\Output{A collection $\sS_L$ of sets of signed features.}
Train an RF using dataset $\cD$ with parameter $\mt$;\
$\sS_G := \{(k, b) \in [p] \times \{-1, 1\} \text{ s.t.\ } \max\limits_{S^\pm \ni (k, b), \abs{S^\pm} \leq s_{\max}} 2^{\abs{S^\pm}} \cdot \DWP_\epsilon(S^\pm) \geq 1 - \threshDWP\}$;\
return $\{(k, b) \in \sS_G \text{ s.t.\ } \PP^*_{\epsilon}(\{(k, b)\}) \geq 1 - \threshPP\}$.
\end{algorithm2e}
\begin{theorem}[Consistency of signed feature importance]\label{thm:consistencyLocalFI}
Consider the same assumptions 
as in Theorem~\ref{thm:PP_finds_BSI}. Let $b(\epsilon)$ be defined as in~\eqref{eq:bepsilon} with $\epsilon > 0$ fixed.

Let $\sS_L$ denote the output of LocalFeatureLSSFind (Algorithm~\ref{Algo:LocalLSSFindFI}), and suppose the thresholds satisfy
\[ 2^s \cdot b(\epsilon) < \threshDWP < \frac{[C_m]^s}2 \quad\text{and}\quad 2^s \cdot b(\epsilon) < \threshPP < [C_m]^s. \]
Then, with probability approaching one as $n\to\infty$, \[\sS_L = \bigcup_{j} S_j^{* -} \cup \bigcup_{j}S_j^{* +},\] where $S_j^{* -}$ and $S_j^{* +}$ denote the basic signed interactions of $\xtest$ as defined in Definition~\ref{def:BSI_xtest}.
\end{theorem}
Theorem~\ref{thm:consistencyLocalFI} establishes that the set of signed features produced by Algorithm~\ref{Algo:LocalLSSFindFI} asymptotically recovers exactly the signed signal features of the test point $\xtest$.
The proof of Theorem~\ref{thm:consistencyLocalFI} follows from the following two propositions.
To this end, define 
\begin{equation}\label{eq:globalsFI}
    \fDWP_\epsilon(k, b) := \max_{S^\pm \ni (k, b), \abs{S^\pm} \leq s_{\max}} 2^{\abs{S^\pm}} \cdot \DWP_\epsilon(S^\pm).
\end{equation}

\begin{proposition}\label{theo:fDWP}
    Suppose that the data $\cD$ is generated from the LSS model in Definition~\ref{def:LSS_model} with constraints \ref{C:uniformity}--\ref{C:sparsity}. Assume that the trees in the RF are CART trees that satisfy assumptions \ref{A:increasing_depth}--\ref{A:no_bootstrap}.

    Let $\threshDWP$ satisfy $2^s \cdot b(\epsilon) < \threshDWP < \frac{[C_m]^s}2$, and consider a signed feature $(k, b)$. If there exists a basic signed interaction $S_j^\pm$ with $(k, b) \in S_j^\pm$, then \[ \fDWP_{\epsilon}(k,b) \geq 1-\threshDWP \] with probability approaching one as $n\to\infty$. Conversely, if no such basic signed interaction exists, then \[ \fDWP_{\epsilon}(k,b) < 1-\threshDWP \] with probability approaching one as $n\to\infty$.
\end{proposition}

\begin{proposition}\label{theo:PP_feature}
    Suppose that the data $\cD$ is generated from the LSS model in Definition~\ref{def:LSS_model} with constraints \ref{C:uniformity}--\ref{C:sparsity}. Assume that the trees in the RF are CART trees that satisfy assumptions \ref{A:increasing_depth}--\ref{A:no_bootstrap}. Fix some test point $\xtest \in [0,1]^p$ independent of $\cD$ such that for all $k \in \cup_{j = 1}^J S_j$ we have $x^*_k \neq \gamma_k$.

    Let $\threshPP$ satisfy $2^s \cdot b(\epsilon) < \threshPP < [C_m]^s$, and consider a signed feature $(k, b)$ that belongs to a basic signed interaction $S^{\pm}_j$. If $S^\pm_j$ is also a BSI for the test point $\xtest$, then \[ \PP^*_{\epsilon}(\{(k, b)\}) \geq 1-\threshPP \] with probability approaching one as $n\to\infty$.  If it is not a BSI for the test point, then \[ \PP^*_{\epsilon}(\{(k, b)\}) < 1-\threshPP \] with probability approaching one as $n\to\infty$.
\end{proposition}
The proofs of these propositions are deferred to the appendix. 

\begin{proof}[Proof of Theorem~\ref{thm:consistencyLocalFI}]
It suffices to show the following:
If there exists a basic signed interaction $S_j^{*\pm}$ for the test point $\xtest$ with $(k, b) \in S_j^{*\pm}$, then
    \[ \fDWP_{\epsilon}(k,b) \geq 1-\threshDWP \quad\text{and}\quad \PP^*_{\epsilon}({(k, b)}) \geq 1-\threshPP \]
    with probability approaching one as $n\to\infty$. Conversely, if no such basic signed interaction for the test point exists, then
    \[ \fDWP_{\epsilon}(k,b) < 1-\threshDWP \quad\text{or}\quad \PP^*_{\epsilon}({(k, b)}) < 1-\threshPP \]
    with probability approaching one as $n\to\infty$.

    If there is a basic signed interaction for the test point, then it is also a basic signed interaction in the LSS model. Therefore, the lower bound for $\fDWP$ follows from Proposition~\ref{theo:fDWP} and the lower bound for $\PP^*$ follows from Proposition~\ref{theo:PP_feature}.

    If no basic signed interaction for the test point contains $(k, b)$, we have to distinguish two cases. If there is no basic signed interaction containing $(k, b)$ in general, then the upper bound for $\fDWP$ follows from Proposition~\ref{theo:fDWP}. On the other hand, if there is a basic signed interaction, which is not a basic signed interaction for the test point, then the upper bound for $\PP^*$ follows from Proposition~\ref{theo:PP_feature}.
\end{proof}

\section{Simulations and application}\label{sec:simulation}
We implemented LocalLSSFind in Python and \textsf{R}. The Python-implementation extends and improves upon the implementation of \cite{behr_provable_2022} and works with Decision Trees and RFs from scikit-learn \cite{scikit-learn}. The \textsf{R}-implementation works with Decision Trees and RFs from the \texttt{ranger} package \cite{wright2017ranger}. 

\subsection{Simulated data from LSS models}
To illustrate our theoretical results and assess the empirical performance of LocalLSSFind, we generated data from the LSS model in~\eqref{eq:E_lssmodel} with different parameters and evaluated the recovered interactions. The models considered have $p=20$ features, where each $X_j$ is sampled independently and uniformly from $[0, 1]$, and outcomes were generated by
\begin{equation}\label{eq:sim_lssmodel}
    Y = \sum_{j=1}^{J} \prod_{k=(j-1) \cdot L + 1}^{j \cdot L} \1(X_k < \tau) + \mathcal{N}(0, \sigma^2),
\end{equation}
where $J \in \{1, 2\}$ is the number of basic signed interactions in the LSS model and $L \in \{2, 3, 4\}$ denotes their respective size. The noise variance $\sigma^2$ was chosen such that the signal-to-noise ratio (SNR) is $\operatorname{SNR} \in \{0.5, 1.0, 2.0, 5.0\}$. The threshold was set to $\tau = 0.5^{1/L}$, chosen so that each BSI was a BSI for approximately half of the considered test points, i.e., such that~\eqref{eq:BSI_xtest} holds. 
RF were trained with $n=1,000$ or $n=10,000$ samples using $\mt = p/2 = 10$ (see the discussion in \cite{behr_provable_2022} on this choice) and $500$ trees. LocalLSSFind was applied with $100$ new test points sampled uniformly from $[0, 1]^p$, with the chosen hyper-parameters $\epsilon = \threshDWP = \threshPP = 0.01$ fixed without any tuning.

Considering only the respective BSIs in the LSS model, we compared the path prevalences $\PP$ of interactions that are BSIs for the test point with those that are not (Figure~\ref{fig:PPforBSI}). The observed $\PP$ for BSIs of the respective test point tend to be higher---often close to 1---whereas those for BSIs that are not BSIs for the test point are mostly close to 0. This agrees with the statements of Proposition~\ref{theo:BSIlowerBound} and Proposition~\ref{theo:nonBSIto0}. When comparing different LSS models, the BSIs for the test points are more clearly separated from the remaining BSIs in models with a single interaction than in models with two interactions, and this separation becomes stronger as the sample size increases, as expected.

\begin{figure}[tb]
    \centering
    \includegraphics[width=\textwidth]{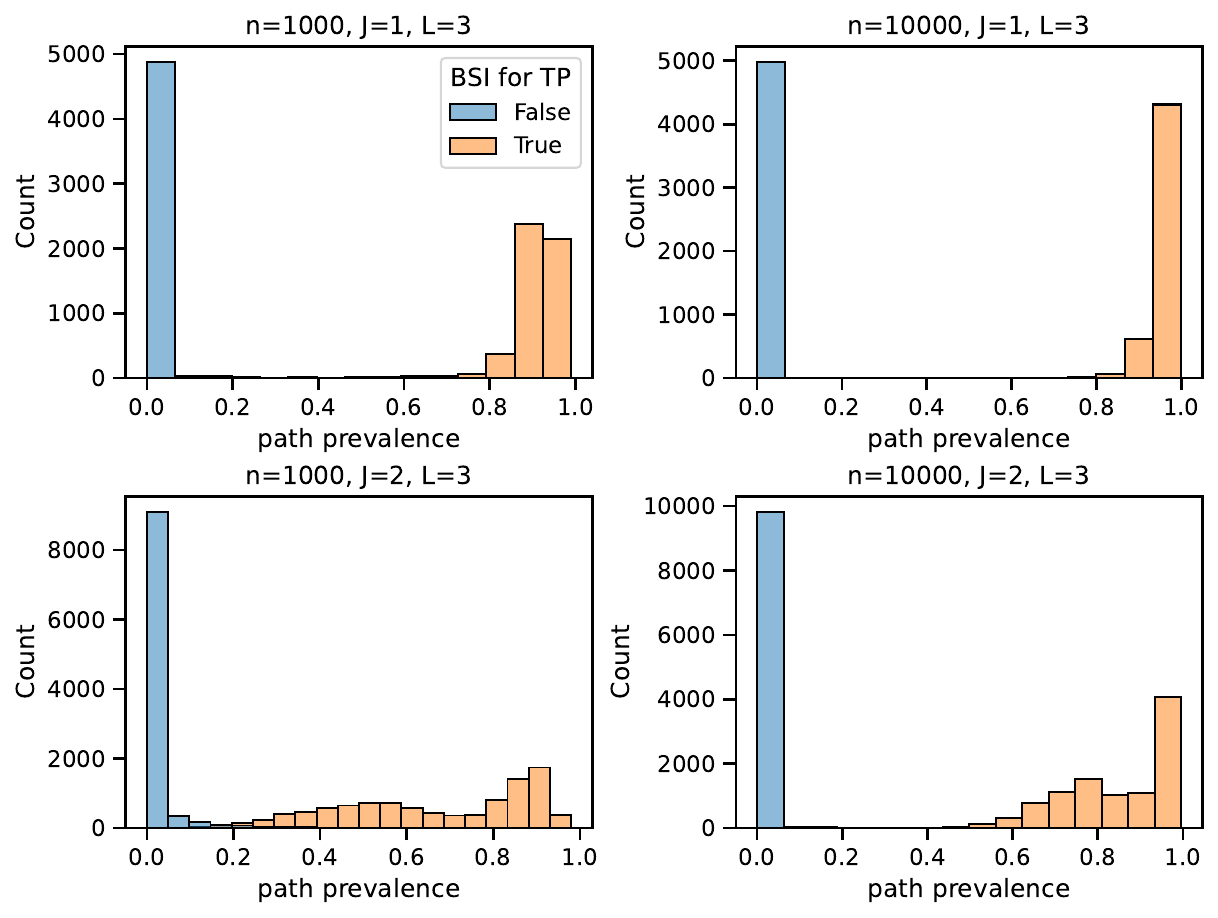}
    \caption{Histograms of $\PP$ for BSIs in the LSS model. BSIs for the test points are shown in orange while BSIs in the LSS model, which are not BSIs for the test points, are shown in blue. In all four cases, \( \operatorname{SNR} = 1.0 \) was used.}
    \label{fig:PPforBSI}
\end{figure}

Because all BSIs for the test points are also BSIs in the LSS model, they are expected to receive high scores even when considering only the global $\DWP$ measure (i.e., independent of the specific test point).
We summarized $\PP$ and $\DWP$ into a single prevalence-based interaction importance ($\PII$) statistic using their normalized product, that is, for any candidate signed interaction $S^\pm \in [p]\times \{-1,1\}$, we defined
\begin{equation}\label{eq:product}
    \PII(S^{\pm}) = 2^{\abs{S^\pm}} \cdot \DWP(S^\pm) \cdot \PP(S^\pm).
\end{equation}
We then compared the rankings of interactions based on $\PII$ with those based on $\DWP$ alone, in order to assess whether including local path-prevalence information improves the identification of interactions that are relevant for the test point $\xtest$.
For models with a maximum interaction size of $L \in \{2,3,4\}$, we considered candidate interactions of size up to $L + 1$. In total, approximately $(3 p)^{L + 1} = 60^{L+1}$ signed candidate interactions are possible, which corresponds to $\approx 10^5, 10^7, 10^9$ for $L = 2,3,4$, respectively.
We considered two statistics to compare the ranking performance of the global $\DWP$ vs.\ the local $\PII$ statistic. First, we evaluated the inclusion rate among the 10 highest-valued signed interactions. Second, we computed a ROC-AUC restricted to the 10 highest-ranked interactions (classifying these as BSIs for the test point), assigning a ROC-AUC of 0 if not all BSIs for the test point were included within the considered interactions.
As the ranking performance is only meaningful for test points $\xtest$ that have at least a single BSI, we restricted the evaluation to such test points.
Note that when $J = 1$, the global and local BSIs coincide. Therefore, Figure~\ref{fig:compare_DWP_PP} shows the results for $J = 2$ only.

\begin{figure}[tb]
    \centering
    \begin{minipage}[t]{0.48\textwidth}
        \centering
        \includegraphics[width=\textwidth]{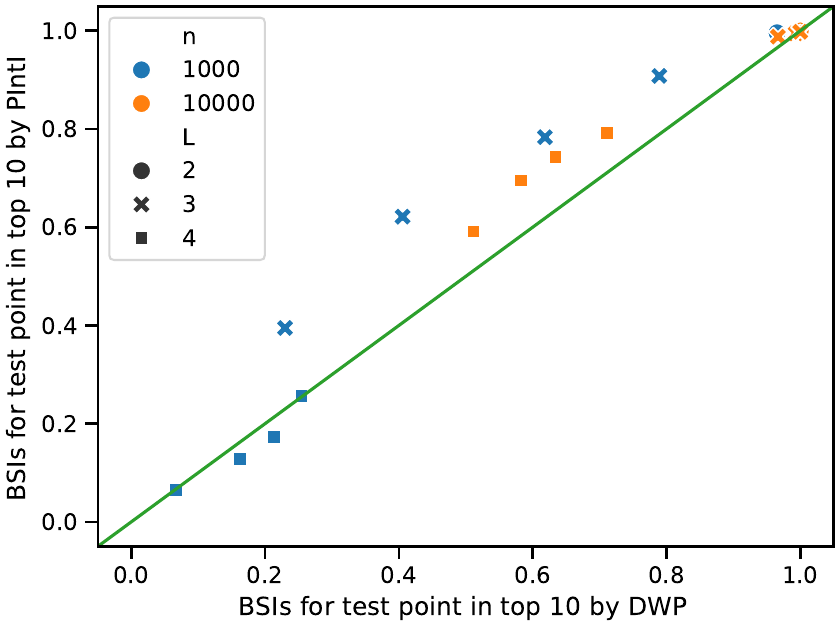}

        \small (a) Relative frequency with which BSIs for the test point are included in the top 10 interactions according to the considered metrics.
    \end{minipage}
    \hfill
    \begin{minipage}[t]{0.48\textwidth}
        \centering
        \includegraphics[width=\textwidth]{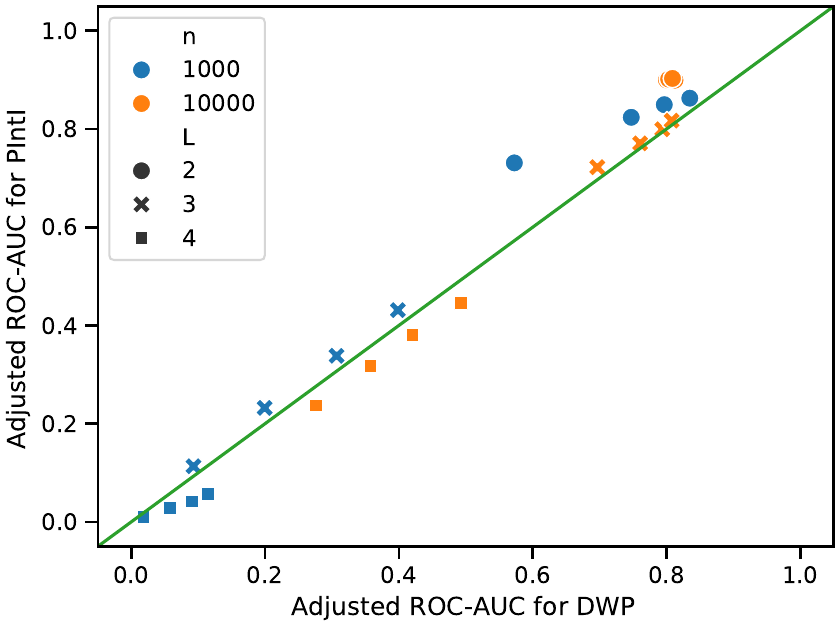}

        \small (b) Adjusted ROC-AUC (see text) based on the ranking of the 10 interactions with highest scores according to the respective metrics.
    \end{minipage}
    \vspace{1em}
    \caption{Evaluations of interaction rankings: y-axis shows rankings based on $\PII$, and x-axis shows rankings based on $\DWP$ alone. Colors and shapes of points indicate simulation parameters $n$ (number of observations) and $L$ (size of interactions), respectively. For each combination of $n,L$ the four different points correspond to different SNRs (see the appendix for the exact values). Points on the green line have equal rankings for both metrics; points above the line indicate better performance with $\PII$, while points below indicate better performance when using $\DWP$ alone.}
    \label{fig:compare_DWP_PP}
\end{figure}

A trade-off can be observed when comparing the global and local rankings (Figure~\ref{fig:compare_DWP_PP}). The local information provided by the $\PP$ statistic focuses specifically on decision paths in the tree ensemble that contain the test point $\xtest$. Consequently, it is primarily influenced by the subset of training observations that share a substantial portion of their decision path with $\xtest$. This reduced effective sample size for local information can negatively impact the overall power of the approach to detect global BSIs and, therefore, also the local BSIs for $\xtest$. For $L = 4$, decision paths must be relatively long to cover the full interaction length. As a result, the number of training observations sharing such long paths is small, and incorporating local information via $\PII$ may lead to worse performance compared to using global information from all decision paths (and thus all training observations) via $\DWP$. However, once the global ranking achieves sufficient accuracy---e.g., ROC-AUC $> 0.5$ or BSIs among the top 10 interactions with a frequency of at least $40\,\%$---we consistently observe an improvement in local rankings when incorporating local path prevalence information.

\subsection{Comparison with TreeSHAP}
An established method for local feature importance for tree ensembles is TreeSHAP by \cite{lundberg_local_2020}. It is also possible to compute TreeSHAP scores for interactions \cite{muschalik_beyond_2024}. 
Although SHAP values are theoretically grounded through their correspondence to a decomposition of the prediction model, this does not guarantee that the recovered features and interactions align with those locally relevant for test points as defined by the LSS model. While LSS models provide a canonical definition of the underlying local signal (including signed interactions), this structure does not necessarily coincide with the prediction-based perspective underlying SHAP values, as we demonstrate below.
Since SHAP values cannot recover the sign of a feature within an interaction, signs are ignored for treeSHAP in the following simulation study. That is, an interaction is counted as a true positive for treeSHAP as long as the correct interaction (without sign) is recovered, whereas for LocalLSSFind the sign must also be correctly reconstructed to count as a true positive.

We again simulated data from LSS models as before. Figure~\ref{fig:ranks_prod_SHAP} compares the rankings of global BSIs (and BIs, respectively) in the LSS model based on treeSHAP scores and the LocalLSSFind score via $\PII$. (Signed) interactions that are actual BSIs (or BIs) for the test point $\xtest$ are highlighted in yellow, while BSIs (or BIs) that are global but not local for $\xtest$ are shown in blue. The figure clearly demonstrates that treeSHAP fails to capture local interactions for the test point, in contrast to LocalLSSFind.
Figure~\ref{fig:frequency_prod_SHAP} reports the frequency with which the top-10 ranked interactions based on treeSHAP and LocalLSSFind contain the local BSIs of the test point. Again, these results confirm that treeSHAP does not identify local BSIs within the LSS model, in contrast to LocalLSSFind.

\begin{figure}[tb]
    \centering
    \includegraphics[width=\textwidth]{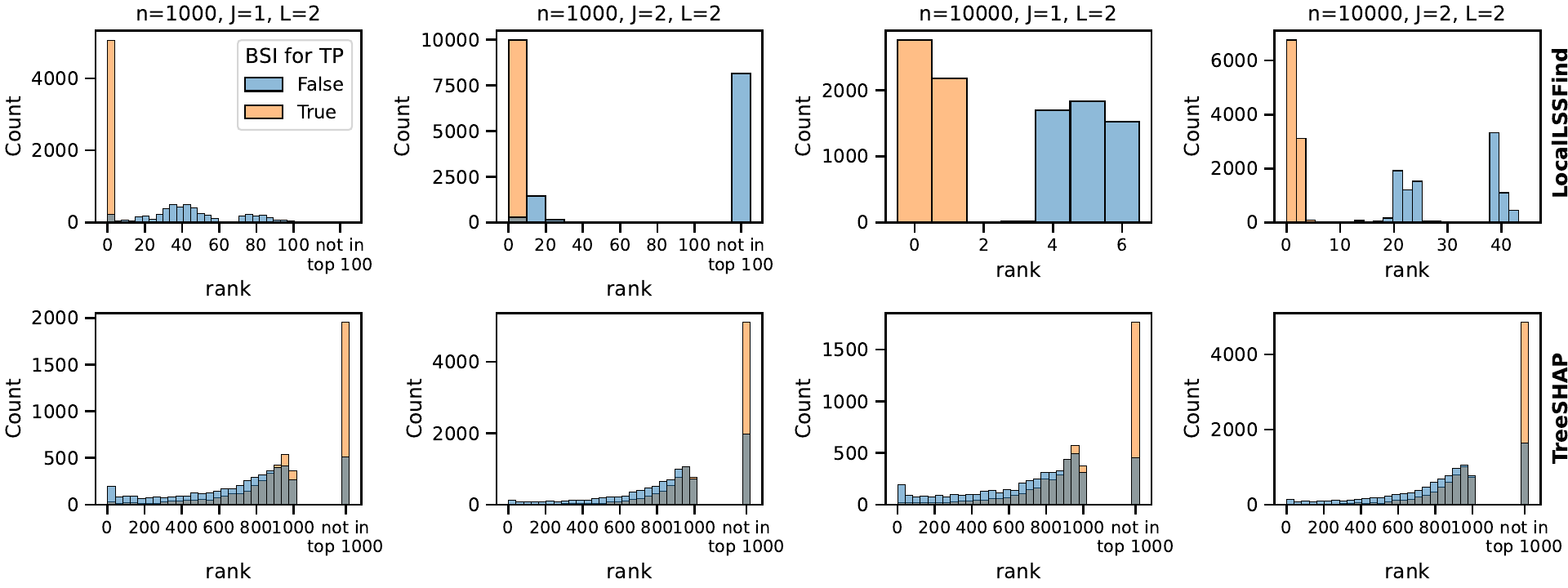}
    \caption{Ranks of local interaction importance scores of BSIs in the LSS model. BSIs for the test points are shown in orange while BSIs, that are not BSIs for the test points, are shown in blue. The simulation parameters are shown above the graphs. In all cases, \( \operatorname{SNR} = 1.0 \) was used. Top row: LocalLSSFind via $\PII$; bottom row: TreeSHAP.}
    \label{fig:ranks_prod_SHAP}
\end{figure}

\begin{figure}[tb]
    \centering
    \includegraphics[width=0.5\textwidth]{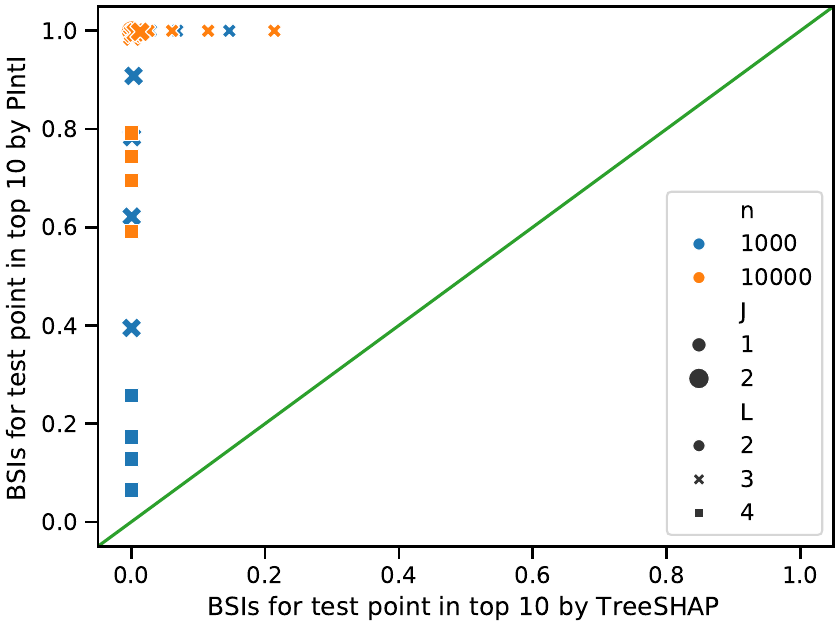}
    \caption{Relative frequency, with which BSIs for the test point are included in the top 10 interactions based on LocalLSSFind (via $\PII$) (y-axis) and TreeSHAP (x-axis).}
    \label{fig:frequency_prod_SHAP}
\end{figure}

\subsection{Application to COMPAS dataset} \label{sec:application}

We analyzed the publicly available COMPAS dataset of defendants from Broward County, Florida, for whom COMPAS violent recidivism scores were recorded \cite{larson2016we}. To reduce sparsity in some categories, we restricted the sample to African-American and Caucasian defendants and followed the preprocessing procedure described by ProPublica and adopted in \cite{herbinger2024decomposing},  
resulting in 3,373 observations. The outcome of interest was the binary indicator derived from the COMPAS violent recidivism score and distinguishes high (= 1) and low (= 0) recidivism risk. The considered features were age (\emph{age}), number of prior offenses (\emph{priors}), an indicator of whether the current charge is a felony (\emph{crime}, $1$ corresponds to felony and $2$ to misdemeanors), ethnicity (\emph{ethnicity}, $1$ corresponds to African-American and $2$ to Caucasian), and gender (\emph{gender} $1$ corresponds to female and $2$ to male). 
We employed the \texttt{ranger} implementation of RF in \textsf{R} \cite{wright2017ranger}, using the same hyperparameter settings as in the previous simulation study: $500$ trees and $ \threshDWP = \threshPP = 0.01$. Here, we used a different impurity decrease threshold, $\epsilon =0.001$, than in the simulation study ($\epsilon =0.01$), since the choice of $\epsilon$ depends on the scale of the response variable; results for $\epsilon =0.01$ are provided in the appendix.
Additional model tuning was performed via 10-fold cross-validation using the \texttt{tuneRanger} package \cite{probst2019hyperparameters}, optimizing the ROC-AUC.
The parameter $\mt$ was selected as $\mt = 3$ or $\mt = 2$, depending on the fold (note that the theoretically motivated optimal choice from \cite{behr_provable_2022} is $p/2 = 2.5$). The minimum node size (\texttt{min.node.size}) was also determined through cross-validation.
The overall cross-validated ROC-AUC for predicting the COMPAS score with this RF model was approximately $0.82$.
When evaluating LocalLSSFind, we used the respective hold-out fold to obtain the test points $\xtest$.

\begin{figure}[tb]
    \centering
    \includegraphics[width=\linewidth]{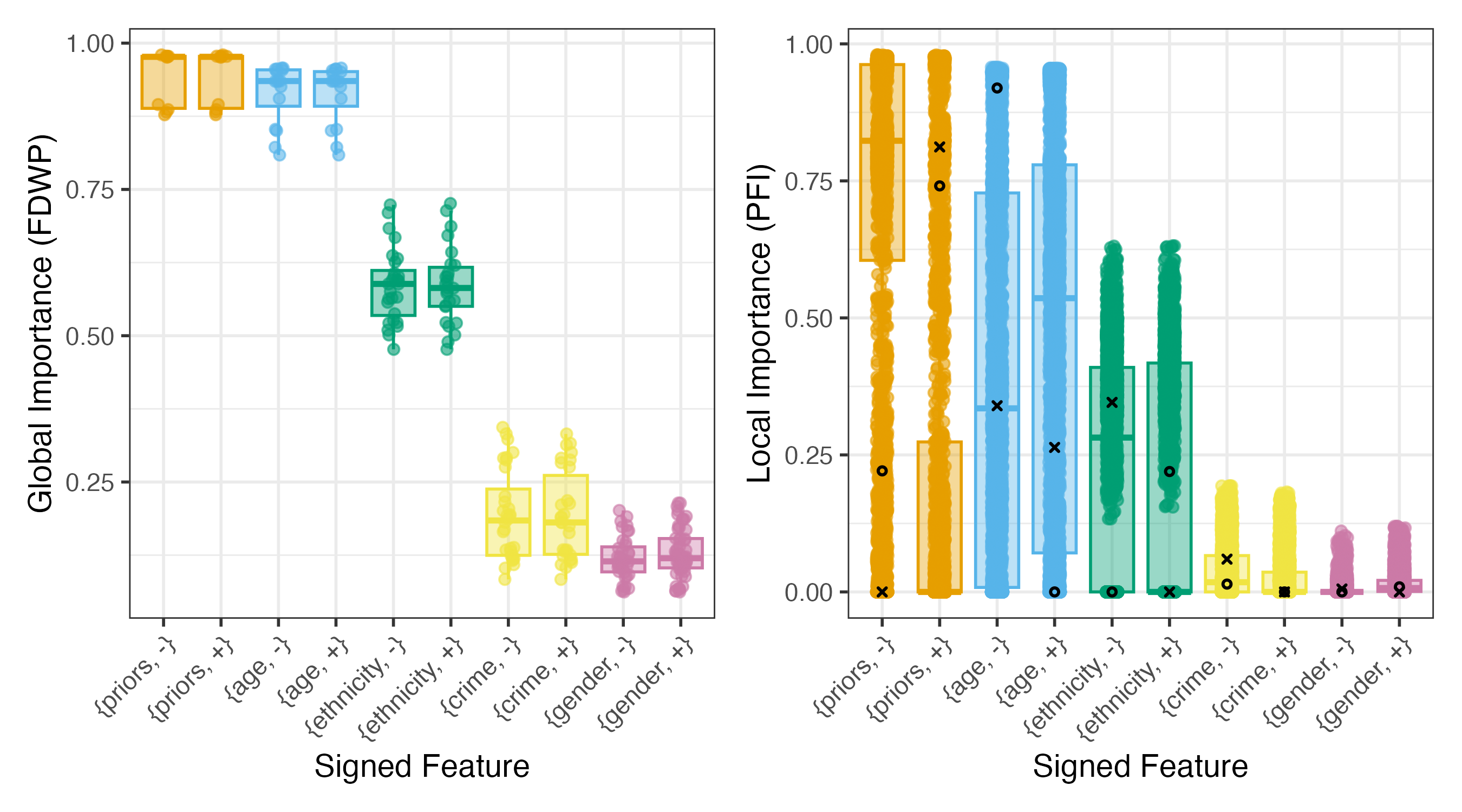}
    \caption{Global signed feature importance scores, $\fDWP$ as in~\eqref{eq:globalsFI}, across ten fold models (left) and local signed feature importance scores, $\PFI$ as in~\eqref{eq:localPFI}, for all observations. In the right panel, the circle and cross correspond to two individuals with the same predicted probability of $84\,\%$ for receiving a class one COMPAS score, see main text for details.}
    \label{fig:fi}
\end{figure}

Figure~\ref{fig:fi} (left) shows the global signed feature importance (cf.\ \eqref{eq:globalsFI}) over the 10 different folds of the cross validation scheme. The overall magnitude is consistent with the classical mean decrease in Gini impurity (MDI) implemented in the \texttt{ranger} package (which gave $202.50$, $196.57$, $40.12$, $9.05$, and $8.26$ as feature importance for features \emph{priors}, \emph{age}, \emph{ethnicity}, \emph{crime}, and \emph{gender}, respectively)%
\footnote{We observed that all binary features (\emph{ethnicity}, \emph{crime}, and \emph{gender}) exhibit lower importance compared to the two continuous-valued features (\emph{priors} and \emph{age}) in this example. It is well known that feature importance measures in RF can be biased toward features with a larger number of distinct values; see, e.g., \cite{zhou2020}. To assess whether this bias occurs here, we additionally applied the de-biased MDI proposed by \cite{zhou2020}, but this adjustment did not substantially alter the overall MDI magnitudes. 
Furthermore, introducing a small amount of additional uniform noise to the binary features did not increase their importance significantly.}.
For the global signed importance scores, we observed minimal differences between the two signs. In general, differences in global signed feature importance between the two signs can arise when a feature interacts with another feature within a specific sign combination. 

For local signed feature importance, we considered the same product statistic as in~\eqref{eq:product}, but for individual features instead of interactions, as in Algorithm~\ref{Algo:LocalLSSFindFI}, namely
\begin{equation}\label{eq:localPFI}
   \PFI(k,b) =   \PP_\epsilon(\{(k,b)\}) \cdot \fDWP_\epsilon(k, b), \text{ for } (k,b) \in [p] \times \{-1,1\}.
\end{equation}
Figure~\ref{fig:fi} (right) shows the distribution of these local signed feature importance scores across all subjects. As shown, there are significant differences in features that are considered important for different subjects.
For illustration, we highlighted two specific subjects in the right part of Figure~\ref{fig:fi}. 
The first subject is a 23-year-old Caucasian male with four prior offenses, charged with a felony (represented by a circle). 
The second subject is a 35-year-old African-American woman with thirteen prior offenses and the same charge (represented by a cross). 
Both individuals received the same predicted probability of a high violent recidivism COMPAS score, namely $84\,\%$. 
As suggested by Figure~\ref{fig:fi} (right), for the African-American female, the high predicted risk seems to be primarily driven by her large number of prior offenses. 
For the Caucasian male, the high predicted risk seems primarily influenced by young age and large number of prior offenses. 
Given that he had only four prior offenses---moderate to low compared to the overall population---it is plausible that this is driven by an interaction effect with his young age.
To investigate interaction effects, we first considered all second-order signed interactions and their global $\DWP$ importance scores. We observed that signed interactions between age and prior offenses are the largest among all pairwise interactions, which is consistent with the findings in \cite{herbinger2024decomposing}. Figure~\ref{fig:interaction} shows an interaction map of pairwise signed feature interaction scores. 

 \begin{figure}[tb]
 \centering
    \includegraphics[width=0.5\textwidth]{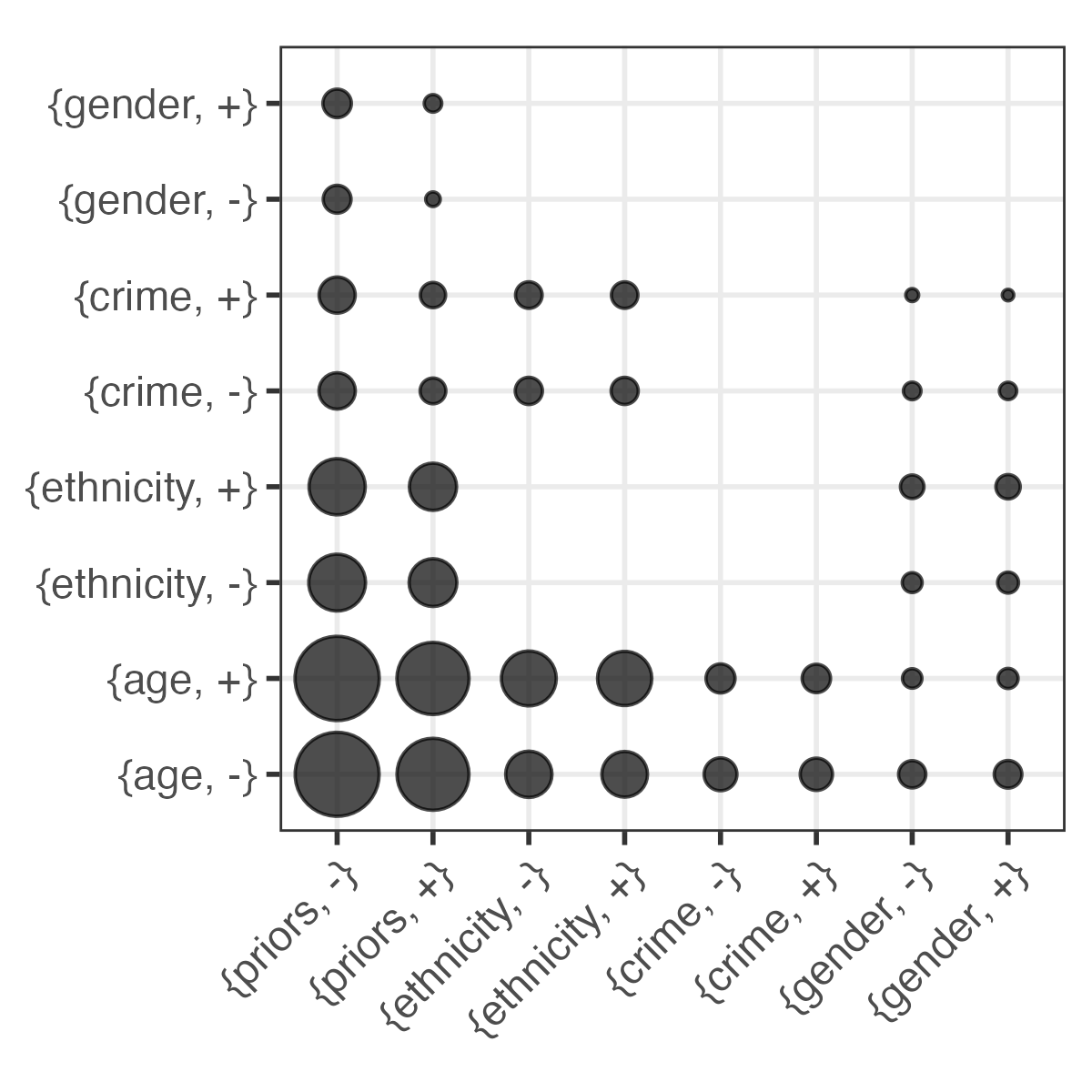}
         \caption{Interaction map showing pairwise signed feature interaction scores for the COMPAS data application (impurity decrease threshold $\epsilon=0.001$). Each point represents an interaction between two signed features, with point size reflecting $\DWP$ importance (scaled as $-\log_{10}(1 - \DWP)$).}
         \label{fig:interaction}
\end{figure}

Figure~\ref{fig:pred_age_priors} presents the corresponding local signed interaction scores as in~\eqref{eq:product} for all observations, together with the predicted probability of a high COMPAS score. 
The same two subjects are highlighted again in this figure. 
For the Caucasian male (highlighted as a circle), the signed interaction between young age and large number of prior offenses receives a high interaction importance score. 
This suggests that for him, the combination of his young age with his (moderate in the overall population) number of prior offenses was the primary driver of his high predicted COMPAS score.
For the African-American female, none of the four signed interactions between age and prior offenses stands out specifically, suggesting that it is mainly her overall large number of prior offenses (also large in the overall population) that is driving her high predicted COMPAS score.
\begin{figure}[tb]
\centering
    \includegraphics[width=\textwidth]{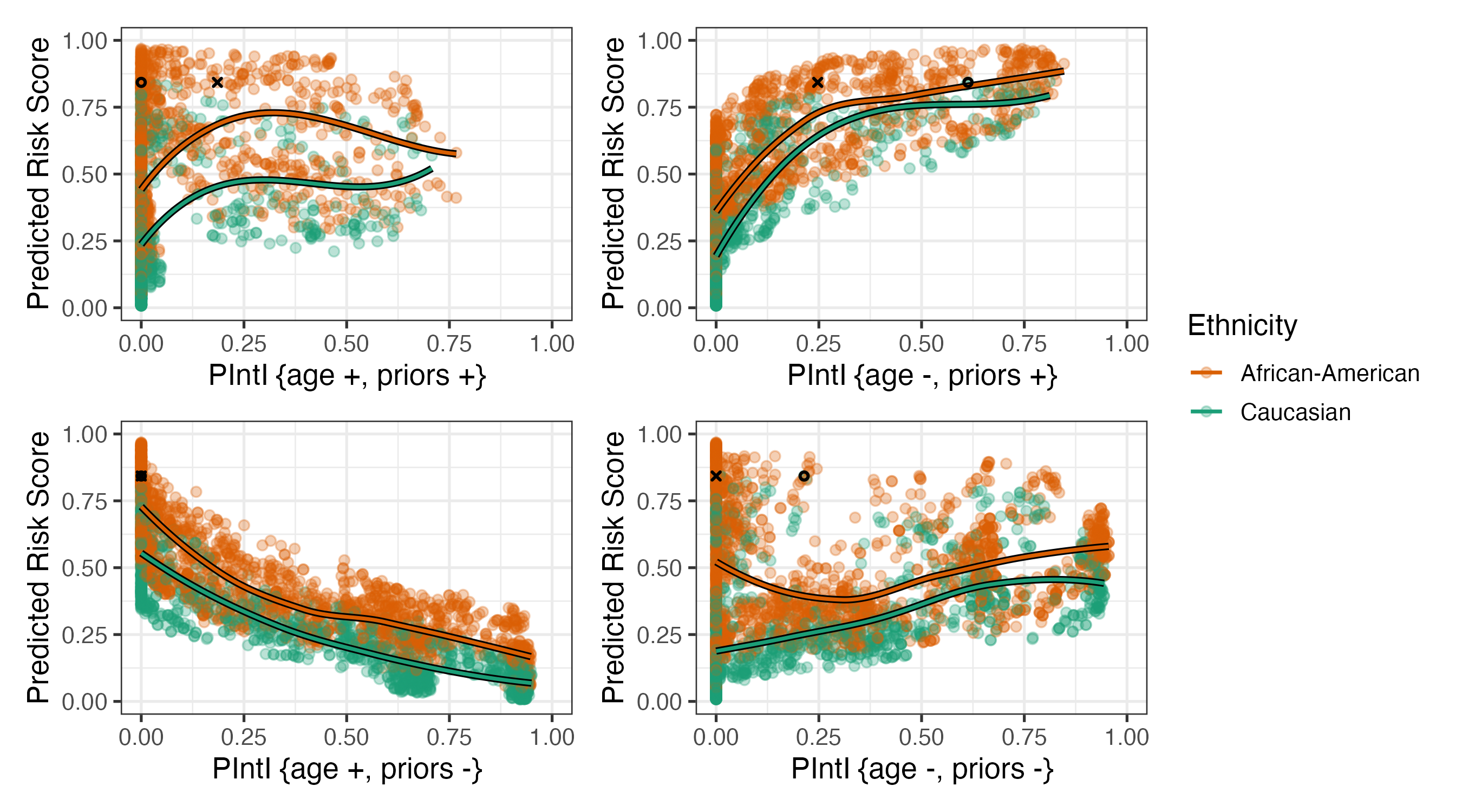}
         \caption{Relationship between local interaction importance ($\PII$) of the signed age--priors interactions and the predicted risk score, stratified by ethnicity. LOESS curves are shown for each group. The circle and cross correspond to two individuals with the same predicted COMPAS score ($84\,\%$), see main text for details.}
         \label{fig:pred_age_priors}
\end{figure}

\section{Discussion}

In this paper, we propose a new methodology to obtain local signed FII scores for RF. 
The LSS model assumption allows us to study the consistency of local signed feature and interaction recovery, and we show that our method provably recovers those signed features and interactions that are relevant for an individual's outcome. 
Our approach is model-specific and explores frequent co-occurrences of signed features along individual decision paths in the tree ensemble, building on prior work on model-specific signed interaction scores \cite{kumbier2018, basu_iterative_2018, behr_provable_2022}.

In simulations, we confirm our theoretical results and demonstrate that other interaction importance metrics---primarily driven by overall prediction accuracy, such as Shapley values---can lead to misleading interpretations of local interactions within the context of the LSS model. 
In contrast, our new methodology reliably recovers qualitative information about which signed features and their interactions primarily drive an individual's prediction, independent of marginal signal strength, which heavily influences prediction differences and thus scores such as Shapley values.
Finally, in an application to the COMPAS dataset, we show that our local signed scores provide valuable insights for individuals regarding which features and interactions, together with their directionality, are most important for their personal outcome.

We note that our theoretical analysis relies heavily on the LSS model assumption, which may be regarded as a limitation. 
On the other hand, interaction behavior is intrinsically connected to scaling and, therefore, to certain modeling assumptions of the regression function. 
Without such assumptions, any function can essentially be expressed as an additive function, in which no interaction behavior is present. 
Nevertheless, in practical applications, interactions often play a significant role, especially at the local level---recall our discussion of the interaction between a large number of prior offenses and young age for the Caucasian male in the COMPAS data example in Section~\ref{sec:simulation}. 
The LSS model assumption makes explicit which types of local, signed interactions are targeted by our method, and thus facilitates interpretation in practice.

Moreover, our theoretical analysis of LocalLSSFind assumes i.i.d.\ covariates. As in LSSFind \cite{behr_provable_2022}, performance is expected to degrade as feature correlations increase. One possible approach to address this issue is to apply feature decorrelation techniques (e.g., \cite{frohlich2025decorrelated}). Extending our results to LSS models with correlated features remains an interesting direction for future research.

\section*{Acknowledgments}
 The project was funded by the Deutsche Forschungsgemeinschaft (DFG, German Research Foundation), project number 509149993, TRR 374, and by the Bavarian Research Institute for Digital Transformation. The authors gratefully acknowledge the Leibniz Supercomputing Centre for funding this project by providing computing time on its Linux-Cluster. The authors also acknowledge funding of the Bavarian Californian Technology Center (BaCaTeC). The authors would like to thank Prof.\ Bin Yu for helpful discussions on this project.

\printbibliography

\appendix

\section{Implementation}

\paragraph{Python module:} A Python implementation of LocalLSSFind is provided as part of the \texttt{lssfind} package and is available at \url{https://github.com/behr-group/lssfind}.

\paragraph{Code repository:} The \textsf{R} implementation of LocalLSSFind (and LSSFind), together with scripts for the simulations and the COMPAS data application, is available at \\\url{https://git.uni-regensburg.de/behr-group-public/locallssfind}

\paragraph{Data:} The COMPAS data set used for the illustration of LocalLSSFind in Section~\ref{sec:application} is from ProPublica and available at \url{https://github.com/propublica/compas-analysis/}.

\section{Notations} \label{sec_sup:notations}
In this section, we list the main notation used throughout the paper.
\begin{itemize}
    \item $\cD=\{(\x_1,y_1), \dots, (\x_n,y_n)\}$ stands for labeled data with $\x_i =(x_{i1}, \dots , x_{ip}) \in \R^p$ and labels $y_i \in \R$.
    \item $n$ denotes the number of samples.
    \item $p$ denotes the number of features. 
    \item $s$ denotes the number of signal features.
    \item $\xtest = (x^*_1, \dots , x^*_p) \in \R^p$ is a new observation.
    \item $\beta_1, \dots, \beta_J$ are coefficients in the LSS model (Definition~\ref{def:LSS_model}).
    \item $\gamma_k$ are thresholds in the LSS model (Definition~\ref{def:LSS_model}).
    \item $C_{\beta}>0,~ C_{\gamma} \in (0,0.5) $ are constants. 
    \item $S_1, \dots, S_J$ are sets of features (Basic Interactions (BIs)). 
    \item $S^\pm_1, \dots, S^\pm_J$ are sets of signed features (Basic Signed Interactions (BSIs)) with $S^\pm_j = \{(k,b_k) : k \in S_j,~ b_k \in \{0,1\}\}$. 
    \item $S^*_j$ is a BI for $\xtest$ and $S^{*\pm}_j$ is a BSI for $\xtest$ with $j \in \{1, \dotsc, J\}$ (cf.\ Definition~\ref{def:BSI_xtest} above).
    \item $S^{*\pm}$ is a Union Signed Interaction (USI) for $\xtest$ (cf.\ Definition~\ref{def:USI_xtest} above).
    \item $T$ denotes a tree in an RF. 
    \item $\cP$ denotes a path in a tree which consists of a sequence of nodes $t \in \{1,\dots,d, t_{\leaf} \}$, where $d$ represents the depth of the path and $t_{\leaf}$ is a leaf node. 
    \item $\ptest$ denotes the path taken by $\xtest$. 
    \item $\theta_t$ denotes the splitting threshold used for node $t$ in a tree $T$. 
    \item $U(t)$ denotes the desirable feature set (cf.\ Definition~\ref{def:desirable_features} above). 
    \item $\Omega_0(\cP)$ denotes the event that the desirable features are exhausted at the leaf node of path $\cP$.
    \item $\dot \parent^\pm(t)$ denotes the set of signed, while $\dot \parent(t)$ denotes the set of unsigned features used by the parents of node $t$ in $T.$ 
    \item $\parent^\pm(t) \subseteq \dot \parent^\pm(t)$ and $\parent(t) \subseteq \dot \parent(t)$, where $\parent^\pm(t)$ and $\parent(t)$ include only the signed or unsigned feature corresponding to the first occurrence of each feature along the path from the root to $t$, if a feature appears multiple times.
    \item $\cF(\cP)$ is the desirable signed feature set of $F^\pm(t_{\leaf})$ (cf.\ Definition~\ref{def:F(P)} above).
    \item $R_t$ denotes the hyper-rectangle in the feature space corresponding to node $t$ in a tree $T$.
    \item $R_{t,l}(k,\theta)$ and $R_{t,r}(k,\theta)$ denote the hyper-rectangles obtained by splitting $R_t$ along feature $k$ at threshold $\theta$.
    \item $\Delta_I(t)$ and $\Delta^n_I(t)$ denote the population and the finite-sample impurity decrease\footnote{The definition of impurity decrease in \cite{behr_provable_2022}, equation~(6), contains a typo: it omits a factor $\frac{N_n(t)}{n}$, which penalizes nodes with fewer samples. However, in Lemma~S2 of its supplement, the correct version---identical to the one used here---is applied.}, respectively. $\Delta_I(t)$ is only defined and used in the proof of Lemma~\ref{thm:AstarEpsilonTo1}.
    \item $\hat{\cF}_{\epsilon}(\cP, T,\cD)$ is the signed feature set corresponding to splits along $\cP$ with at least impurity decrease $\epsilon$.
    \item $\PP^*_{\epsilon}(S^{\pm})$ is the path prevalence of $S^{\pm}$ on $\ptest$ with minimum impurity decrease $\epsilon$.
\end{itemize}

\section{Preliminaries} \label{sec_sup:preliminaries}
We define different feature sets within the context of a decision tree $T$ in an RF. Each path $\cP$ in the tree $T$ consists of a sequence of nodes $t \in \{ 1,\dots,d, t_{\leaf}\}$, where $d$ represents the depth of the path, and $t_{\leaf}$ is a leaf node. Along this path a sequence of signed features $(k_{1},b_{1}),\dots,(k_{d},b_{d})$ is associated, where $k_{t} \in \{1,\dots, p\}$ indicates the feature index and $b_{t}\in \{-1,+1 \}$ indicates the direction of the split for that feature at node $t.$ Here, $b_{t} = -1$ denotes a split that follows the $\leq$ direction, while $b_{t} = +1$ denotes a split that follows the $>$ direction. For each inner node $t$, $\theta_t$ denotes the splitting threshold used for that node. 

Moreover, for each node $t$ in the tree $T$, we define the following sets:
\begin{itemize}
    \item $\dot \parent^\pm(t)$ is the set of signed features used by the parents of node $t$ in $T.$ 
    \item $\dot \parent(t)$ is the corresponding set of unsigned features used by the parents of node $t.$
    \item $\parent^\pm(t)$ is a subset of $\dot \parent^\pm(t)$ that includes only the signed feature corresponding to the first occurrence of each feature along the path, if a feature appears multiple times. 
    \item $\parent(t)$ is the set of unsigned features corresponding to $\parent^\pm(t)$.
\end{itemize}

We now define the \emph{desirable} feature set $U(t),$ which consists of all features that would lead to a positive decrease in impurity if the RF model could observe the full distribution $\P(X,Y)$ (with respect to some particular LSS model).

\begin{definition}[$U(t),$ desirable feature set] \label{def:desirable_features}
Define the desirable feature set $U(t) \subset [p]$ to be 
\[ U(t)=\{ k \in [p] : \exists \, j \in [J] \text{ s.t. } k \in S_j, S_j^+ \cap F^{\pm}(t) = \emptyset \text{ and } (k,-1) \notin F^{\pm}(t) \}.\]
\end{definition}

Note that, since we assume the LSS model uses only $\leq$ signs, the condition $S_j^+ \cap F^{\pm}(t) = \emptyset$ in the above definition ensures that, on the path to $t$, no split on a feature from $S_j$ has been taken in the wrong direction.

We define the event $\Omega_0(\cP)$ to be that the desirable features are exhausted at the leaf node of the path $\cP$. More precisely: 
\begin{definition}[The event $\Omega_0(\cP)$]\label{def:omega_0}
        \[ \Omega_0(\cP) = \{ U(t_{\leaf}) = \emptyset \text{ for the leaf node } t_{\leaf} \text{ of } \cP \}.\]
\end{definition}
Next, we define the desirable signed feature set for a path $\cP$.
\begin{definition}[$\cF(\cP),$ desirable signed feature set of $F^\pm(t_{\leaf})$] \label{def:F(P)}
Define the set $\cF(\cP) \subset [p] \times \{-1, +1\}$ as 
\begin{multline*}
    \cF(\cP) = \{(k_t, b_t) \in F^\pm(t_{\leaf}) : k_t \in U(t), \text{ where } t \text{ is an inner node of } \cP,\\
    \text{and } t_{\leaf} \text{ is a leaf node of } \cP \}.
\end{multline*}
\end{definition}
The feature set $\cF(\cP)$ defined above is an oracle feature set because it depends on the true interactions $S_j$, which are not known in practice. However, a consistent estimate of $\cF(\cP)$ can be obtained by thresholding on the mean decrease in impurity (see Section B.1 in the supplement of \cite{behr_provable_2022}), which leads to the definition of $\hat{\cF}_{\epsilon}(\cP, T,\cD)$.

Analogously to the definition of a union signed interaction in \cite{behr_provable_2022}, we define the union signed interaction for $\xtest$ as the union of one or more individual BSIs for $\xtest$. 
\begin{definition}[Union Signed Interaction (USI) for $\xtest$] \label{def:USI_xtest}
$S^{*\pm}$ is a union signed interaction for $\xtest$, if 
\[S^{*\pm}= \bigcup_{j\in \cI_{-}} S_j^{*-} \cup \bigcup_{j \in \cI_{+}} S_j^{*+} ,\]
for some set of indices 
\begin{gather*}
\cI_{-} \subset \{j \in [J] : S_j^{-} \text{ is BSI for } \xtest\}, \\
\cI_{+} \subset \{j \in [J] : S_j^{+} \text{ is BSI for } \xtest \}.
\end{gather*}
\end{definition}
Note that for all $j\in \cI_{+}$ we only have single-feature interactions, i.e., $\abs{S_j^*}=1$. In contrast to the LSS model, for each single-feature BI $S_j^*$ only one of the signs ($-1$ or $+1$) can appear in USIs for $\xtest$.

\section{The Population Case} \label{sec_sup:pop_case}
Recall the notation for a new test point $\xtest = (x_1^*, \dots, x_p^*)$. Define the constant $C^* > 0$ such that for all signal features $k \in S_j, j=1, \dotsc, J$
\[ C^* \leq \abs{x_k^* - \gamma_k}, \]
where $\gamma_k$ denotes the threshold for feature $k$ in the LSS model. Define the event $A^*(\cD, T)$ to ensure that for every node $t$ along the path $\ptest(\cD,T)$, where feature $k_t$ is desirable, the threshold $\theta_t$ at node $t$ remains within an interval around the true threshold $\gamma_{k_t}:$
\begin{equation}\label{eq:A^*}
    A^*(\cD, T) = \Big\{\theta_t \in \Big(\gamma_{k_t} - \frac{C^*}{3}, \gamma_{k_t} + \frac{C^*}{3} \Big) \text{ for all } t \in \ptest(\cD,T) \text{ with } k_t \in U(t) \Big\}.
\end{equation}
We also define 
\begin{equation}
    \begin{aligned}\label{eq:A^*_epsilon}
    A^*_{\epsilon}(\cD, T) = \Big\{&\theta_t \in \Big(\gamma_{k_t} - \frac{C^*}{3}, \gamma_{k_t} + \frac{C^*}{3} \Big) \\ & \text{ for all } t \in \ptest(\cD,T) \text{ with } k_t \in U(t)
     \text{ and } \Delta_I^n(t) \geq \epsilon \Big\}.
\end{aligned}
\end{equation}
Under the assumption that $\theta_t \in \Big(\gamma_{k_t} - \frac{C^*}{3}, \gamma_{k_t} + \frac{C^*}{3} \Big)$, $\theta_t$ and $\gamma_{k_t}$ are on the same side with respect to $x_k^*$. In the following, for notational convenience, we write $A^*_{\epsilon}$ and $A^*$ instead of $A^*_{\epsilon}(\cD, T)$ and $A^*(\cD, T)$, respectively.
\begin{lemma} \label{thm:USI_prob}
Assume that $A^*$ from above holds true.
If $S^{*\pm}$ is a USI for $\xtest$ as in Definition~\ref{def:USI_xtest}, then, for any data $\cD$ and any decision tree $T$, we have that $ \Omega_0(\ptest)$ implies $S^{*\pm} \subset \cF(\ptest) $.
\end{lemma} 

\begin{proof}
Consider any fixed decision tree $T$. We want to show that if $S^{*\pm} \not\subset \cF(\ptest) $ then $\Omega_0^c(\ptest)$ occurs, given that the event $A^*$ holds true. 

Assume that $S^{*\pm} \not\subset \cF(\ptest)$, i.e., there exists at least one element $(k, b) \in S^{*\pm}$ that is not in $\cF(\ptest)$. Since $S^{*\pm}$ is a USI for $\xtest$, it follows that there exists some BI $S_j^*$ for $\xtest$ such that $k \in S_j^*$. In the following, we will distinguish between the cases that the corresponding sign $b$ is $-1$ (case (i)) and that it is $+1$ (case (ii)).

\begin{enumerate}
\renewcommand{\labelenumi}{(\roman{enumi})}
    \item Let $(k, b) = (k, -1) \in S_j^- \subset S^{*\pm}$. First, assume that $(k, -1) \in F^\pm(t^*_{\leaf})$. Then, for the respective node of $\ptest$ with $(k, -1) = (k_t, -1)$, because $(k, -1) \notin \cF(\ptest)$, it follows from the definition of $\cF(\ptest)$ that $k_t \notin U(t)$. Thus, from the definition of $U(t)$, it must follow that $S^+_j \cap F^\pm(t) \neq \emptyset$.

On the other hand, if $(k, -1) \notin F^\pm(t^*_{\leaf}) \land k \notin U(t^*_{\leaf})$, the definition of $U(t^*_{\leaf})$ requires $S^+_j \cap F^\pm(t^*_{\leaf}) \neq \emptyset$ too. In this case set $t = t^*_{\leaf}$.

Among all elements in $S^+_j \cap F^\pm(t)$, we consider the signed feature which appears first on $\ptest$. We denote the respective node as $\tilde{t}$ and the signed feature as $(k_{\tilde{t}}, +1)$.
For this feature, it holds that $(k_{\tilde{t}}, -1) \in U(\tilde{t})$. As $(k_{\tilde{t}}, +1)$ appears on $\ptest$, it must follow that
\[ x^*_{k_{\tilde{t}}} > \theta_{\tilde{t}}. \]
The definition of $C^*$ indicates that $x^*_{k_{\tilde{t}}}$ could be either $x^*_{k_{\tilde{t}}} \leq \gamma_{k_{\tilde{t}}} - C^*$ or $x^*_{k_{\tilde{t}}} \geq \gamma_{k_{\tilde{t}}} + C^*$. Because $(k_{\tilde{t}}, -1) \in S_j^{*-}$ we must have $x^*_{k_{\tilde{t}}} \leq \gamma_{k_{\tilde{t}}}$, so only $x^*_{k_{\tilde{t}}} \leq \gamma_{k_{\tilde{t}}} - C^*$ is possible. Then
\[ x^*_{k_{\tilde{t}}} \leq \gamma_{k_{\tilde{t}}} - C^* < \theta_{\tilde{t}} + \frac{C^*}{3} - C^* < \theta_{\tilde{t}}, \]
where the second inequality holds because $\abs{\theta_{\tilde{t}} - \gamma_{k_{\tilde{t}}}} < \frac{C^*}{3},$ which follows from the definition of $A^*$. This implies $x^*_{k_{\tilde{t}}} < \theta_{\tilde{t}}$, which is a contradiction to $x^*_{k_{\tilde{t}}} > \theta_{\tilde{t}}$.

Thus, because both $(k, -1) \in F^\pm(t^*_{\leaf})$ and $(k, -1) \notin F^\pm(t^*_{\leaf}) \land k \notin U(t^*_{\leaf})$ lead to a contradiction, we must have $k \in U(t^*_{\leaf})$. Then $\Omega_0^c(\ptest)$ holds true by definition.

\item If the sign $b$ is $+1$, then $k$ is in a single-feature signed interaction $S_j^{*+} = \{(k, +1)\}$. Assume that $(k, +1) \in F^\pm(t^*_{\leaf})$. Using the same arguments as before, for the corresponding node $t$ of $\ptest$ with $(k, +1) = (k_t, +1)$, it follows that $k_t \notin U(t)$. This is a contradiction because for a single-feature signed interaction, it always holds that $k_t \in U(t)$, as by definition there was no split on the feature $k_t$ on the path before and there are no other features in the corresponding BSI $S_j^+$. Thus, $(k, +1) \notin F^\pm(t^*_{\leaf})$ holds and $k \in U(t^*_{\leaf})$ follows, causing $\Omega_0^c(\ptest)$ to hold true.
\end{enumerate}
\end{proof}

\section{The Finite Sample Case} \label{sec_sup:sample_case}
Recall assumptions \ref{A:increasing_depth}--\ref{A:no_bootstrap} and constraints \ref{C:uniformity}--\ref{C:sparsity} above. Define the families of trees $\cT_1, \cT_2$ as in \cite{behr_provable_2022}:
\begin{align*}
\cT_1 &\triangleq \{\text{Any CART tree that satisfies~\ref{A:balancedsplit} and~\ref{A:no_bootstrap}}\}, \\
\cT_2 &\triangleq \{\text{Any CART tree that satisfies~\ref{A:balancedsplit},~\ref{A:no_bootstrap}, and~\ref{A:mtry}}\}.
\end{align*}

\begin{lemma}\label{thm:AstarEpsilonTo1}
Assume that $T\in \cT_1$ and constraints \ref{C:uniformity}--\ref{C:sparsity} hold. Then, for any fixed $\epsilon>0$,
\[ \P_T(A^*_{\epsilon} \mid \cD) \pto 1, \text{ as } n \to \infty. \]
\end{lemma}

\begin{proof}
Define 
\[ \bestTheta_{t,k} = \argmax_{\theta \in [C_{\gamma}, 1-C_{\gamma}]} \Delta_I^n(R_{t,l}(k, \theta), R_{t,r}(k, \theta)). \]
Recall the definition of $A^*_{\epsilon}$ in \eqref{eq:A^*_epsilon}. Since $T\in \cT_1$, each split in $T$ is constructed according to the CART algorithm, which selects the threshold that maximizes the finite-sample impurity decrease defined in equation~\eqref{eq:impdecrease} above. Therefore, the threshold $\theta_t$ at any node $t$ satisfies $\theta_t = \bestTheta_{t,k_t}$ and we can replace $\theta_t$ in the definition of $A^*_{\epsilon}$ by $\bestTheta_{t,k_t}$. Define the events
\begin{multline*}
    \tilde{A}^*_{\epsilon} := \Big\{\bestTheta_{t,k} \in \Big(\gamma_{k} - \frac{C^*}{3}, \gamma_{k} + \frac{C^*}{3} \Big) \\\text{for all } t \in \ptest(\cD,T) \text{ with } U(t) \neq \emptyset \text{ and } \Delta_I^n(t) \geq \epsilon \text{ and for all } k\in U(t)\Big\}
\end{multline*}
and
\begin{multline*}
    \tilde{A}_{\epsilon} := \Big\{\bestTheta_{t,k} \in \Big(\gamma_{k} - \frac{C^*}{3}, \gamma_{k} + \frac{C^*}{3} \Big) \\\text{for all } t \in T \text{ with } U(t) \neq \emptyset \text{ and } \Delta_I^n(t) \geq \epsilon \text{ and for all } k\in U(t)\Big\}.
\end{multline*}
$A^*_{\epsilon}$ requires that the chosen split variable $k_t \in U(t)$ yields a threshold $\bestTheta_{t,k_t}$ close to $\gamma_{k_t}$, whereas $\tilde{A}^*_{\epsilon}$ requires closeness of $\bestTheta_{t,k}$ to $\gamma_k$ for all $k\in U(t)$, even if they were not chosen. Thus, $\tilde{A}^*_{\epsilon} \subseteq A^*_{\epsilon}.$
Since $\tilde{A}_{\epsilon}$ requires the same condition as $\tilde{A}^*_{\epsilon}$ but over all relevant nodes in the entire tree $T$, we have $\tilde{A}_{\epsilon} \subseteq \tilde{A}^*_{\epsilon}.$ Hence,
\begin{equation*}
    \P_T(A^*_{\epsilon}\mid \cD) \geq  \P_T(\tilde{A}^*_{\epsilon}\mid \cD) \geq  \P_T(\tilde{A}_{\epsilon}\mid \cD).
\end{equation*}
Thus, in order to prove that $\P_T(A^*_{\epsilon} \mid \cD) \pto 1$, it suffices to show that  $\P_T(\tilde{A}^*_{\epsilon} \mid \cD) \pto 1 \text{ as } n \to \infty.$ We can express $\P_T(\tilde{A}_{\epsilon}\mid \cD)$ in a maximum-based formulation as follows 
\begin{equation}\label{prob:B_epsilon}
    \P_T(\tilde{A}_{\epsilon}\mid \cD) 
= \P_T\left( \max_{\substack{t \in T \\ \Delta_I^n(t) \geq \epsilon \\ U(t) \neq \emptyset}} \max_{k \in U(t)} \abs{\bestTheta_{t,k} - \gamma_k} < \frac{C^*}{3} \,\middle|\, \cD \right) .
\end{equation}

Now, define the \emph{population impurity decrease}\footnote{Similar to the finite-sample impurity decrease in equation~\eqref{eq:impdecrease} of the main text, the definition of the population impurity decrease in equation~(29) of the supplement of~\cite{behr_provable_2022} contains a typo. Instead, we use the formula from equation~(30) in the same document, which also forms the basis for subsequent results.} as 
\begin{equation}\label{eq:population-impurity-dec}
    \Delta_I(t) = \Delta_I(R_{t, l}, R_{t, r}) = \frac{\mu(R_{t, l}) \mu(R_{t, r})}{\mu(R_t)} [\E(Y \mid X \in R_{t, l}) - \E(Y \mid X \in R_{t, r})]^2,
\end{equation}
and define the event
\[ B_n = \left\{\sup_{R_{t, l}, R_{t, r} \in \sR} \abs{\Delta_I^n(R_{t, l}, R_{t, r}) - \Delta_I(R_{t, l}, R_{t, r})} \leq \epsilon/2\right\}, \]
which is independent of the tree $T$. By Proposition S6 (b) from the supplement of \cite{behr_provable_2022}, we have
\[ \sup_{R_{t, l}, R_{t, r} \in \sR} \abs{\Delta_I^n(R_{t, l}, R_{t, r}) - \Delta_I(R_{t, l}, R_{t, r})} \pto 0 \text{ as } n\to\infty. \]
Therefore, as $n\to\infty$,
\begin{align*}
    \P_\cD(B_n) &= \P_\cD\left(\sup_{R_{t, l}, R_{t, r} \in \sR} \abs{\Delta_I^n(R_{t, l}, R_{t, r}) - \Delta_I(R_{t, l}, R_{t, r})} \leq \epsilon/2\right) \\
    &= 1- \P_\cD \left( \sup_{R_{t, l}, R_{t, r} \in \sR} \abs{\Delta_I^n(R_{t, l}, R_{t, r}) - \Delta_I(R_{t, l}, R_{t, r})} > \epsilon/2 \right) \to 1.
\end{align*}

On the event $B_n$, we have 
\[ \abs{\Delta_I^n(t) - \Delta_I(t)} \leq \epsilon/2 \quad \text{for all } t \in T. \]
Hence, if $\Delta_I^n(t) \geq \epsilon$, it follows that
\[ \Delta_I(t) \geq \Delta_I^n(t) - \abs{\Delta_I^n(t) - \Delta_I(t)} \geq \epsilon - \epsilon/2 = \epsilon/2. \]
By the bounded-response constraint \ref{C:bounded-response}, we have $\abs{Y} < 1$. This implies 
\[ \abs{\E(Y \mid X \in R)} < 1 \]
for any region $R$. Therefore, for any region $R_{t,l}, R_{t,r}$
\begin{equation}\label{eq:bound-response-regions-4}
    [\E(Y \mid X \in R_{t, l}) - \E(Y \mid X \in R_{t, r})]^2 < 4.
\end{equation}
Moreover, since the hyper-rectangles of the child nodes partition the parent region, we have $\mu(R_{t, l}) + \mu(R_{t, r}) = \mu(R_t)$, and the product $\mu(R_{t, l}) \mu(R_{t, r}) / \mu(R_t)$ is maximized when both children have equal size (i.e.\ $\mu(R_{t, l}) = \mu(R_{t, r}) = \mu(R_{t})/2$). Thus,
\[ \frac{\mu(R_{t, l}) \mu(R_{t, r})}{\mu(R_t)} \leq \frac{1}{4} \mu(R_t). \]
Using this together with \eqref{eq:bound-response-regions-4}, we obtain with the definition of $\Delta_I(t)$ in \eqref{eq:population-impurity-dec}
\[ \Delta_I(t) \leq \frac{1}{4} \mu(R_t) \cdot 4 = \mu(R_t). \]
Therefore, on the event $B_n$, if $\Delta_I^n(t) \geq \epsilon$, then
\[ \Delta_I(t) \geq \epsilon/2 \quad \Rightarrow \quad \mu(R_t) \geq \epsilon / 2 =: \tilde{\epsilon}. \]
Thus, on the event $B_n$, 
\[ \{ t: \Delta^n_I(t) \geq \epsilon \} \subseteq \{ t: \mu(R_t) \geq \tilde{\epsilon}\}.\]
For the probability of $\P_T(\tilde{A}_{\epsilon} \mid \cD)$ in \eqref{prob:B_epsilon}, we now obtain:
\begin{align*}
    & \P_T(\tilde{A}_{\epsilon} \mid \cD) \\
    & = \P_T\left( \max_{\substack{t\in T \\ \Delta_I^n(t) \geq \epsilon \\ U(t)\neq \emptyset}} \max_{k\in U(t)} \abs{\bestTheta_{t,k} - \gamma_k} < \frac{C^*}{3} \,\middle|\, \cD \right) \\
    &\geq \P_T\left(\left\{\max_{\substack{t\in T \\ \Delta_I^n(t) \geq \epsilon \\ U(t)\neq \emptyset}} \max_{k\in U(t)} \abs{\bestTheta_{t,k} - \gamma_k} < \frac{C^*}{3}\right\} \cap B_n \,\middle|\, \cD \right) \\
    &\geq \P_T\left(\left\{\max_{\substack{t\in T \\ \mu(R_t) > \tilde\epsilon \\ U(t)\neq \emptyset}} \max_{k\in U(t)} \abs{\bestTheta_{t,k} - \gamma_k} < \frac{C^*}{3}\right\} \cap B_n \,\middle|\, \cD \right) \\
    &\geq \P_T\left(\max_{\substack{t\in T \\ \mu(R_t) > \tilde\epsilon \\ U(t)\neq \emptyset}} \max_{k\in U(t)} \abs{\bestTheta_{t,k} - \gamma_k} < \frac{C^*}{3} \,\middle|\, \cD \right) + \P_T(B_n \mid \cD) - 1.
\end{align*}
We show that the first two terms in the last sum each converge in probability to 1. For the first term, we use Lemma S11 (iii) in the supplement of \cite{behr_provable_2022}. It states that 
\[ \sup_{T\in T_1(\cD)}\max_{\substack{t\in T \\ \mu(R_t) > \tilde\epsilon \\ U(t)\neq \emptyset}} \max_{k\in U(t)} \abs{\bestTheta_{t,k} - \gamma_k} \pto 0 \text{ as } n\to\infty. \]
Consequently, for any fixed constant $c>0$ (in particular for $c=C^*/3$), it follows that
\[ \P_\cD\left( \sup_{T \in T_1(\cD)} \max_{\substack{t\in T \\ \mu(R_t) > \tilde\epsilon \\ U(t)\neq \emptyset}} \max_{k\in U(t)} \abs{\bestTheta_{t,k} - \gamma_k} < c \right) \to 1 \text{ as } n\to\infty. \]
Here, if for some fixed data $\cD$, the bound
\[ \sup_{T \in T_1(\cD)} \max_{\substack{t\in T \\ \mu(R_t) > \tilde\epsilon \\ U(t)\neq \emptyset}} \max_{k\in U(t)} \abs{\bestTheta_{t,k} - \gamma_k} < c \]
holds, then
\[ \P_{T}\left(\max_{\substack{t\in T \\ \mu(R_t) > \tilde\epsilon \\ U(t)\neq \emptyset}} \max_{k\in U(t)} \abs{\bestTheta_{t,k} - \gamma_k} < c \,\middle|\, \cD\right) = 1. \]
Thus, 
\[ \P_T\left(\max_{\substack{t\in T \\ \mu(R_t) > \tilde\epsilon \\ U(t)\neq \emptyset}} \max_{k\in U(t)} \abs{\bestTheta_{t,k} - \gamma_k} < \frac{C^*}{3} \,\middle|\, \cD \right) \pto 1. \]
As $B_n$ depends only on the data $\cD$, it is independent of the randomness in $T$, so 
\[ \P_T(B_n \mid \cD) = \1(B_n). \]
Furthermore, since $\P_\cD(B_n)\to 1$, it follows that 
\[ \1(B_n) \pto 1. \]
Combining these convergence results, we get
\[ \P_T(\tilde{A}_{\epsilon} \mid \cD) \geq \P_T\left(\max_{\substack{t\in T \\ \mu(R_t) > \tilde\epsilon \\ U(t)\neq \emptyset}} \max_{k\in U(t)} \abs{\bestTheta_{t,k} - \gamma_k} < \frac{C^*}{3} \,\middle|\, \cD \right) + \P_T(B_n \mid \cD) - 1 \pto 1 \]
and therefore also, as $n \to \infty$,
\[ \P_T(A^*_{\epsilon} \mid \cD) \pto 1. \]
\end{proof}

\begin{lemma}\label{thm:AstarTo1}
Assume that $T\in \cT_2$ and that constraints \ref{C:uniformity}--\ref{C:sparsity} and assumptions \ref{A:increasing_depth}--\ref{A:no_bootstrap} hold. Then
\[ \P_T(A^* \mid \cD) \pto 1, \text{ as } n \to \infty. \]
\end{lemma}

\begin{proof}
Let $\epsilon > 0$ be arbitrary. If for all nodes $t\in\ptest$ with $k_t\in U(t)$ we have that $\Delta_I^n(t) \geq \epsilon$, then $A^* = A^*_\epsilon$. Define event 
\[ E=\{\forall t\in\ptest \text{ with } k_t\in U(t): \Delta_I^n(t) \geq \epsilon\}. \]
We now have
\begin{align*}
    \P_T(A^* \mid \cD) &\geq \P_T(A^*_{\epsilon} \cap E \mid \cD) \\
    &= \P_T(A^*_\epsilon \mid \cD) + \P_T(E \mid \cD) - \P_T(A^*_\epsilon \cup E \mid \cD) \\
    &\geq \P_T(A^*_\epsilon \mid \cD ) + \P_T(E \mid \cD) - 1.
\end{align*}
By Proposition S13 (iii) in \cite{behr_provable_2022}, for any path $\cP$, we have
\begin{align}\label{eq:split_decrease_min}
    &\P_T\Big( \min_{t\in \cP } \min_{k\in U(t)} \Delta_I^n(R_{t,l}(k, \bestTheta_{t,k}), R_{t,r}(k, \bestTheta_{t,k})) \geq \epsilon \Bigm| \cD \Big)  \\ 
    \geq & 1 - \Big( \frac{4\epsilon}{C^2_{\beta} C_{\gamma}^{2\max_j \abs{S_j}-1}}\Big)^{\tilde{C}} - \eta_n(\cD,\epsilon),
\end{align}
with constant $\tilde{C}=C^{2s}_m/\log(1/C_{\gamma})$ and $\eta_n(\cD, \epsilon) \pto 0 $. Rewriting $E$ and using \eqref{eq:split_decrease_min}, it follows that
\begin{align*}
    \P_T(E \mid \cD) &= \P_T(\forall t\in\ptest \text{ with } k_t\in U(t): \Delta_I^n(t) \geq \epsilon \mid \cD) \\
    &= \P_T\Big( \min_{t\in\ptest \text{ with } k_t\in U(t)} \Delta_I^n(t) \geq \epsilon \Bigm| \cD \Big) \\
    &= \P_T\Big( \min_{t\in\ptest \text{ with } k_t\in U(t)} \Delta_I^n(R_{t,l}(k_t, \bestTheta_{t,k_t}), R_{t,r}(k_t, \bestTheta_{t,k_t})) \geq \epsilon \Bigm| \cD \Big) \\
    &\geq \P_T\Big(\min_{t\in\ptest \text{ with } k_t\in U(t)} \min_{k \in U(t)} \Delta_I^n(R_{t,l}(k, \bestTheta_{t,k}), R_{t,r}(k, \bestTheta_{t,k})) \geq \epsilon \Bigm| \cD \Big) \\
    &\geq \P_T\Big(\min_{t\in\ptest} \min_{k \in U(t)} \Delta_I^n(R_{t,l}(k, \bestTheta_{t,k}), R_{t,r}(k, \bestTheta_{t,k})) \geq \epsilon \Bigm| \cD \Big) \\
    &\geq 1 - \Big( \frac{4\epsilon}{C^2_{\beta} C_{\gamma}^{2\max_j \abs{S_j}-1}}\Big)^{\tilde{C}} - \eta_n(\cD,\epsilon).
\end{align*}
The first inequality holds because, with $k_t \in U(t)$, we also have
\[ \Delta_I^n(R_{t,l}(k_t, \bestTheta_{t,k_t}), R_{t,r}(k_t, \bestTheta_{t,k_t})) \geq \min_{k \in U(t)} \Delta_I^n(R_{t,l}(k, \bestTheta_{t,k}), R_{t,r}(k, \bestTheta_{t,k})). \]
In the second inequality, we extend the set of considered nodes from those that split on a desirable feature (i.e., $t\in\ptest \text{ with } k_t\in U(t)$) to all nodes on $\ptest$ (i.e., $t\in\ptest$).

By Lemma~\ref{thm:AstarEpsilonTo1}, we know that $P(A^*_\epsilon \mid \cD) \pto 1$. Combining this with our previous bound, we get
\begin{align*}
    \P_T(A^* \mid \cD) &\geq \P_T(A^*_\epsilon \mid \cD) + \P_T(E\mid \cD) - 1 \\
    &= \P_T(A^*_\epsilon\mid \cD) - \Big( \frac{4\epsilon}{C^2_{\beta} C_{\gamma}^{2\max_j \abs{S_j}-1}}\Big)^{\tilde{C}} - \eta_n(\cD,\epsilon) \\
    &\pto 1 - \Big( \frac{4\epsilon}{C^2_{\beta} C_{\gamma}^{2\max_j \abs{S_j}-1}}\Big)^{\tilde{C}}.
\end{align*}
Since this holds for any $\epsilon > 0$ and $\Big( \frac{4\epsilon}{C^2_{\beta} C_{\gamma}^{2\max_j \abs{S_j}-1}}\Big)^{\tilde{C}} \to 0$ for $\epsilon \to 0$, we conclude that $\P_T(A^*\mid \cD) \pto 1$.
\end{proof}

\begin{proposition}\label{theo:hatFequalF}
Suppose that assumptions \ref{A:increasing_depth}--\ref{A:no_bootstrap} and constraints \ref{C:uniformity}--\ref{C:sparsity} hold and that $T\in \cT_2$. For any fixed constant $\epsilon > 0$, the following holds true:
\begin{align}
&\P_{T}\left( \Omega_0^c (\ptest) \,\middle|\, \cD \right)  \pto 0;\label{eq:omega}\\
& \P_{T}\left( \hat{\cF}_\epsilon(\ptest) \nsubseteq {\cF(\ptest)} \,\middle|\, \cD \right) \pto 0;\label{eq:not_contain}\\
&\P_{T}\left( \hat{\cF}_\epsilon (\ptest) \neq {\cF(\ptest)} \,\middle|\, \cD \right) \leq \left(\frac{ 4 \epsilon}{C_\beta^2 C_\gamma^{2\max_j \abs{S_j} -1}}\right)^{\tilde{C}} + \eta_n(\cD, \epsilon); \label{eq:not_equal}
\end{align}
with $\tilde{C} = C_m^{2 s} / \log(1/C_\gamma)$ and $\eta_n(\cD, \epsilon)\pto 0$.
\end{proposition}
\begin{proof}
   The proof is identical to that of Theorem S3 in \cite{behr_provable_2022}. Theorem S3 in \cite{behr_provable_2022} considers a random path $\cP$, but the proof is also valid for the fixed path $\ptest$.
\end{proof}

\subsection*{\textbf{Proof of Proposition~\ref{theo:BSIlowerBound} above}}
\begin{proof}
    Let
    \[ r_n(\cD, \epsilon) = \max \left( \P_{T}(\Omega_0^c \mid \cD) +  \eta_n(\cD, \epsilon), \P_{T}(\hat{\cF}_{\epsilon}  \nsubseteq \cF \mid \cD) \right) + \P_T({A^*}^c \mid \cD). \]
    By Lemma~\ref{thm:AstarTo1} and Proposition~\ref{theo:hatFequalF} $r_n(\cD, \epsilon) \pto 0$. Let $b(\epsilon) = \left(\frac{ 4 \epsilon}{C_\beta^2 C_\gamma^{2\max_j \abs{S_j} -1}}\right)^{\tilde{C}}$. We have the following chain of inequalities: 
    \begin{align*}
     & \PP^*_{\epsilon}(S^{*\pm}) \\
     &= \P_{T} (S^{*\pm} \subset \hat{\cF}_{\epsilon}(\ptest,T,\cD) \mid \cD ) \\
     &\geq \P_{T} \left( S^{*\pm} \subset \cF(\ptest) \mid \cD \right) - \P_{T}(\hat{\cF}_{\epsilon}(\ptest,T,\cD) \neq \cF(\ptest) \mid \cD )  \\
     &\geq \P_{T} \left( S^{*\pm} \subset \cF(\ptest) \mid \cD \right) - \left(\frac{ 4 \epsilon}{C_\beta^2 C_\gamma^{2\max_j \abs{S_j} -1}}\right)^{\tilde{C}} - \eta_n(\cD, \epsilon) \\
     &\geq 1- \P_{T}(\Omega_0^c \mid \cD) - \P_T({A^*}^c \mid \cD) - b(\epsilon) - \eta_n(\cD, \epsilon) \\
     &\geq 1 - b(\epsilon) - r_n(\cD, \epsilon).
  \end{align*}
  where the inequality $\P_{T} \left( S^{*\pm} \subset \cF(\ptest) \mid \cD \right) \geq 1- \P_{T}(\Omega_0^c \mid \cD) - \P_T({A^*}^c \mid \cD)$ follows from Lemma~\ref{thm:USI_prob}.
\end{proof}

\subsection*{\textbf{Proof of Proposition~\ref{theo:nonBSIto0} above}}
\begin{proof}
    Since $S^\pm$ is a BSI in the LSS model but not a BSI for $\xtest$, there exists at least one feature $k$ such that either $(k, -1) \in S^\pm$ and $x^*_{k} > \gamma_{k}$, or $(k, +1) \in S^\pm$ and $x^*_{k} \leq \gamma_{k}$. Denote the set of these features as $\mathbb{K}$.
    
    If there is a $k\in\mathbb{K}$, such that no node $t$ with $\Delta_I^n(t) \geq \epsilon$ on $\ptest$ splits on feature $k$, then it follows directly that $S^\pm \nsubseteq \hat{\cF}_{\epsilon}(\ptest, T,\cD)$.
    
    Now, assume that all $k\in\mathbb{K}$ appear on $\ptest$ at nodes with $\Delta_I^n(t) \geq \epsilon$. Then choose the one which first appears on $\ptest$ and let $t \in \ptest$ be the corresponding node. By Lemma~\ref{thm:AstarEpsilonTo1}, we have that $A^*_{\epsilon}$ holds and $\abs{\bestTheta_{t,k} - \gamma_k} \leq \frac{C^*}{3} < \abs{x^* - \gamma_k}$ with probability approaching 1. In this case, $\bestTheta_{t,k}$ and $\gamma_k$ are on the same side relative to $x^*$. Thus, if $x^*_k > \gamma_k$ then also $x^*_k > \bestTheta_{t,k}$ and similarly, if $x^*_k < \gamma_k$, then $x^*_k < \bestTheta_{t,k}$. 
    
    First, consider $x^*_k > \gamma_k$. Then is $x^*_k > \bestTheta_{t,k}$ and so $(k, +1) \in \hat{\cF}_{\epsilon}(\ptest, T,\cD)$. Because for each $k$ only one of $(k, -1)$ and $(k, +1)$ can be an element of $\hat{\cF}_{\epsilon}(\ptest, T,\cD)$, it follows that $(k, -1) \notin \hat{\cF}_{\epsilon}(\ptest, T,\cD)$. At the same time, by construction of $k$ and $x^*_k > \gamma_k$, we have $(k, -1) \in S^\pm$ in this case. Therefore, $S^\pm \nsubseteq \hat{\cF}_{\epsilon}(\ptest, T,\cD)$.
    
    Analogously, for $x^*_k < \gamma_k$, we have $x^*_k < \bestTheta_{t,k}$ and so $(k, -1) \in \hat{\cF}_{\epsilon}(\ptest, T,\cD)$. Because for each $k$ only one of $(k, -1)$ and $(k, +1)$ can be an element of $\hat{\cF}_{\epsilon}(\ptest, T,\cD)$, it follows that $(k, +1) \notin \hat{\cF}_{\epsilon}(\ptest, T,\cD)$. At the same time, by construction of $k$ and $x^*_k < \gamma_k$, we have $(k, +1) \in S^\pm$ in this case. Therefore, $S^\pm \nsubseteq \hat{\cF}_{\epsilon}(\ptest, T,\cD)$ also in this case. 
    
    Thus, if all $k\in\mathbb{K}$ appear on $\ptest$ at nodes with $\Delta_I^n(t) \geq \epsilon$, $\P_T(S^\pm \nsubseteq \hat{\cF}_{\epsilon}(\ptest, T,\cD) \mid \cD) \geq \P_T(A^*_\epsilon \mid \cD) \pto 1$.
    
    Combining these results, we have $\PP^*_{\epsilon}(S^{\pm}) = \P_T(S^{\pm} \subset \hat{\cF}_{\epsilon}(\ptest, T,\cD) \mid \cD) \pto 0$.
\end{proof}

\section{Importance Measures for Single Signed Features}
\label{sec_sup:feature_imp} 

Recall \eqref{eq:globalsFI}:
\begin{equation*}
    \fDWP_\epsilon(k, b) := \max_{S^\pm \ni (k, b), \abs{S^\pm} \leq s_{\max}} 2^{\abs{S^\pm}} \cdot \DWP_\epsilon(S^\pm).
\end{equation*}


\subsection*{\textbf{Proof of Proposition~\ref{theo:fDWP}}}
\begin{proof}
    First consider the case that there is a basic signed interaction $S_j^\pm$ with $(k, b) \in S^\pm$. From the first part of Theorem~2 of~\cite{behr_provable_2022} follows that $ 2^{\abs{S_j^{\pm}}} \cdot \DWP_{\epsilon}(S_j^\pm) \geq 1-\eta_{\DWP}$, so also $\fDWP_{\epsilon}(k, b) \geq 2^{\abs{S_j^{\pm}}} \cdot \DWP_{\epsilon}(S_j^{\pm}) \geq 1-\eta_{\DWP}$, with probability approaching 1 as $n\to\infty$.

    Now consider the case, that there is no basic signed interaction which contains $(k, b)$. Then let $S^\pm$ be the signed interaction containing $(k, b)$ which maximizes $2^{\abs{S^\pm}} \cdot \DWP_{\epsilon}(S^\pm)$. From the second part of Theorem~2 of~\cite{behr_provable_2022} follows $ 2^{\abs{S_j^{\pm}}} \cdot \DWP_{\epsilon}(S_j^\pm) < 1-\eta_{\DWP}$, so also $\fDWP_{\epsilon}(k, b) = 2^{\abs{S_j^{\pm}}} \cdot \DWP_{\epsilon}(S_j^{\pm}) < 1-\eta_{\DWP}$, with probability approaching 1 as $n\to\infty$.
\end{proof}



\subsection*{\textbf{Proof of Proposition~\ref{theo:PP_feature}}}
\begin{proof}
    First consider the case that $S_j^\pm$ is a BSI for the test point. From the definition of path prevalence follows directly, that for any two signed interactions $S_1^\pm \subseteq S_2^\pm$ their path prevalences relate by $\PP^*_{\epsilon}(S_1^\pm) \geq \PP^*_{\epsilon}(S_2^\pm)$. Therefore, with $\{(k, b)\} \subseteq S_j^\pm$, follows $\PP^*_{\epsilon}(k, b) = \PP^*_{\epsilon}(\{(k, b)\}) \geq \PP^*_{\epsilon}(S_j^\pm)$. From Proposition~\ref{theo:BSIlowerBound} follows that $\PP^*_{\epsilon}(S_j^\pm) \geq 1-\eta_{\PP}$ with probability approaching 1 as $n\to\infty$. Combining these two inequalities, we get $\PP^*_{\epsilon}(k, b) \geq \PP^*_{\epsilon}(S_j^\pm) \geq 1-\eta_{\PP}$ with probability approaching 1 as $n\to\infty$.

    If $S_j^\pm$ is not a BSI for the test point, different situation must be considered. First consider the case, where $\xtest$ is on the wrong side of the threshold for feature $k$, i.e. $\{(k, +)\} = S_j^\pm$ with $\xtest_k \leq \gamma_k$ or $(k, -) \in S_j^\pm$ with $\xtest_k > \gamma_k$. As in the proof of Proposition~\ref{theo:nonBSIto0}, this implies $(k, b) \notin \hat{\cF}_{\epsilon}(\ptest, T,\cD)$ with probability approaching 1 as $n\to\infty$, so $\PP^*_{\epsilon}(k, b) \pto 0$.

    If the test point is on the correct side for the considered feature, then there must be a feature $(\bar{k}, -1) \in S_j^\pm$ with $\xtest_{\bar{k}} > \gamma_{\bar{k}}$. Now consider a tree, where the root splits on $\bar{k}$. By Lemma~\ref{thm:AstarTo1} is $\P_T(A^* \mid \cD) \pto 1$ and if $A^*$ holds true, then $\ptest$ follows the $>$ direction of the root split. In this case, for any subsequent node $t$ on $\ptest$ is $(\bar{k}, +1) \in \parent^\pm(t)$ and so $k \notin U(t)$, because $S_j^+ \cap \parent^\pm(t) \neq \emptyset$. This implies also $(k, b) \notin \cF(\ptest)$, so $(k, b) \in \hat{\cF}_\epsilon(\ptest)$ would imply $\hat{\cF}_\epsilon(\ptest) \nsubseteq {\cF(\ptest)}$. But as seen in equation~\eqref{eq:not_contain} of Proposition~\ref{theo:hatFequalF} is $\P_{T}(\hat{\cF}_\epsilon(\ptest) \nsubseteq {\cF(\ptest)} \mid \cD) \pto 0$. We can therefore bound $\PP^*_\epsilon(k, b)$ by
    \begin{align*}
        \PP^*_\epsilon(k, b) &\leq \P_{T}(t_\roo \text{ splits not on }\bar{k} \mid \cD) + \P_T({A^*}^c \mid \cD) + \P_{T}(\hat{\cF}_\epsilon(\ptest) \nsubseteq {\cF(\ptest)} \mid \cD) \\
        &\pto 1 - \P_{T}(t_\roo \text{ splits on }\bar{k} \mid \cD).
    \end{align*}
    By Theorem S2 from the supplement of~\cite{behr_provable_2022}, the probability that the root splits on $\bar{k}$ is almost surely at least $[C_m]^s$ as $n\to \infty$. So in this case, $\PP^*_{\epsilon}(k, b) \leq 1 - [C_m]^s$ with probability approaching 1 as $n\to\infty$.
\end{proof}

\section{Additional Simulation Results and Figures} \label{sec_sup:sim_fig}

\begin{longtable}{llllrrrrr}
\caption{Summary statistics for different importance measures in different simulation settings. Columns $\DWP$, $\PII$, and TreeSHAP show the fraction of simulations with respective settings, in which the ten interactions with the highest importance contained all BSIs for the test point. Columns ROC $\DWP$ and ROC $\PII$ list the average adjusted ROC-AUC for these simulations.}
\label{tab:sims_results}\\
\toprule
n & J & L & SNR & $\DWP$ & $\PII$ & TreeSHAP & ROC $\DWP$ & ROC $\PII$ \\
\midrule
\endfirsthead
\toprule
n & J & L & SNR & $\DWP$ & $\PII$ & TreeSHAP & ROC $\DWP$ & ROC $\PII$ \\
\midrule
\endhead
\midrule
\multicolumn{9}{r}{Continued on next page} \\
\midrule
\endfoot
\bottomrule
\endlastfoot
1000 & 1 & 2 & 0.5 & 0.999803 & 1.000000 & 0.001184 & 0.924136 & 0.934087 \\
1000 & 1 & 2 & 1.0 & 1.000000 & 1.000000 & 0.001973 & 0.951996 & 0.946779 \\
1000 & 1 & 2 & 2.0 & 0.999803 & 1.000000 & 0.004143 & 0.956445 & 0.951755 \\
1000 & 1 & 2 & 5.0 & 1.000000 & 1.000000 & 0.008286 & 0.961290 & 0.956840 \\
1000 & 1 & 3 & 0.5 & 1.000000 & 1.000000 & 0.010915 & 0.932085 & 0.905316 \\
1000 & 1 & 3 & 1.0 & 1.000000 & 1.000000 & 0.029172 & 0.944742 & 0.924765 \\
1000 & 1 & 3 & 2.0 & 1.000000 & 1.000000 & 0.068466 & 0.949593 & 0.930167 \\
1000 & 1 & 3 & 5.0 & 0.999802 & 1.000000 & 0.146458 & 0.948910 & 0.939759 \\
1000 & 1 & 4 & 0.5 & 1.000000 & 0.999800 & 0.000000 & 0.917183 & 0.881243 \\
1000 & 1 & 4 & 1.0 & 1.000000 & 1.000000 & 0.000200 & 0.935787 & 0.905203 \\
1000 & 1 & 4 & 2.0 & 1.000000 & 0.999200 & 0.000200 & 0.941699 & 0.914227 \\
1000 & 1 & 4 & 5.0 & 1.000000 & 0.999000 & 0.000000 & 0.961170 & 0.921784 \\
1000 & 2 & 2 & 0.5 & 0.965241 & 0.995989 & 0.000134 & 0.572500 & 0.731018 \\
1000 & 2 & 2 & 1.0 & 0.999866 & 0.996925 & 0.000000 & 0.747272 & 0.823848 \\
1000 & 2 & 2 & 2.0 & 1.000000 & 0.998529 & 0.000134 & 0.796635 & 0.849578 \\
1000 & 2 & 2 & 5.0 & 0.999866 & 0.998797 & 0.000401 & 0.834783 & 0.862758 \\
1000 & 2 & 3 & 0.5 & 0.229683 & 0.394884 & 0.000266 & 0.092918 & 0.113106 \\
1000 & 2 & 3 & 1.0 & 0.405409 & 0.621503 & 0.000533 & 0.199257 & 0.231954 \\
1000 & 2 & 3 & 2.0 & 0.618305 & 0.783373 & 0.000933 & 0.306804 & 0.338179 \\
1000 & 2 & 3 & 5.0 & 0.789235 & 0.908074 & 0.003331 & 0.398398 & 0.431047 \\
1000 & 2 & 4 & 0.5 & 0.066268 & 0.065206 & 0.000000 & 0.018341 & 0.010300 \\
1000 & 2 & 4 & 1.0 & 0.162550 & 0.127490 & 0.000000 & 0.057503 & 0.028375 \\
1000 & 2 & 4 & 2.0 & 0.212749 & 0.172908 & 0.000000 & 0.090837 & 0.041391 \\
1000 & 2 & 4 & 5.0 & 0.254449 & 0.256839 & 0.000000 & 0.114948 & 0.057290 \\
10000 & 1 & 2 & 0.5 & 1.000000 & 1.000000 & 0.000608 & 0.922218 & 0.950142 \\
10000 & 1 & 2 & 1.0 & 1.000000 & 1.000000 & 0.001216 & 0.926812 & 0.950817 \\
10000 & 1 & 2 & 2.0 & 1.000000 & 1.000000 & 0.003040 & 0.931676 & 0.951808 \\
10000 & 1 & 2 & 5.0 & 1.000000 & 1.000000 & 0.004662 & 0.937621 & 0.956785 \\
10000 & 1 & 3 & 0.5 & 0.999800 & 1.000000 & 0.025738 & 0.941363 & 0.956660 \\
10000 & 1 & 3 & 1.0 & 1.000000 & 1.000000 & 0.060854 & 0.952957 & 0.955662 \\
10000 & 1 & 3 & 2.0 & 0.999800 & 1.000000 & 0.114725 & 0.948634 & 0.957169 \\
10000 & 1 & 3 & 5.0 & 0.999800 & 1.000000 & 0.213687 & 0.947349 & 0.958322 \\
10000 & 1 & 4 & 0.5 & 1.000000 & 1.000000 & 0.000000 & 0.955715 & 0.956510 \\
10000 & 1 & 4 & 1.0 & 1.000000 & 1.000000 & 0.000000 & 0.950190 & 0.958013 \\
10000 & 1 & 4 & 2.0 & 1.000000 & 1.000000 & 0.000000 & 0.949571 & 0.955494 \\
10000 & 1 & 4 & 5.0 & 0.999801 & 1.000000 & 0.000000 & 0.943096 & 0.951715 \\
10000 & 2 & 2 & 0.5 & 0.999867 & 0.999867 & 0.000000 & 0.812699 & 0.899505 \\
10000 & 2 & 2 & 1.0 & 0.999867 & 0.999601 & 0.000000 & 0.799904 & 0.899716 \\
10000 & 2 & 2 & 2.0 & 0.999867 & 0.999734 & 0.000133 & 0.803194 & 0.901573 \\
10000 & 2 & 2 & 5.0 & 0.999867 & 0.999601 & 0.000133 & 0.808913 & 0.902999 \\
10000 & 2 & 3 & 0.5 & 0.966614 & 0.988076 & 0.001457 & 0.696502 & 0.722337 \\
10000 & 2 & 3 & 1.0 & 0.992978 & 0.994701 & 0.002650 & 0.760505 & 0.770759 \\
10000 & 2 & 3 & 2.0 & 0.999868 & 0.996555 & 0.006492 & 0.793636 & 0.799542 \\
10000 & 2 & 3 & 5.0 & 0.999868 & 0.998013 & 0.013249 & 0.807662 & 0.816993 \\
10000 & 2 & 4 & 0.5 & 0.511430 & 0.591308 & 0.000000 & 0.275618 & 0.236768 \\
10000 & 2 & 4 & 1.0 & 0.582403 & 0.695109 & 0.000000 & 0.357394 & 0.317994 \\
10000 & 2 & 4 & 2.0 & 0.634104 & 0.743620 & 0.000000 & 0.420305 & 0.380314 \\
10000 & 2 & 4 & 5.0 & 0.710925 & 0.792398 & 0.000000 & 0.492691 & 0.445394 \\
\end{longtable}

\begin{figure}[tb]
    \centering
    \includegraphics[width=\linewidth]{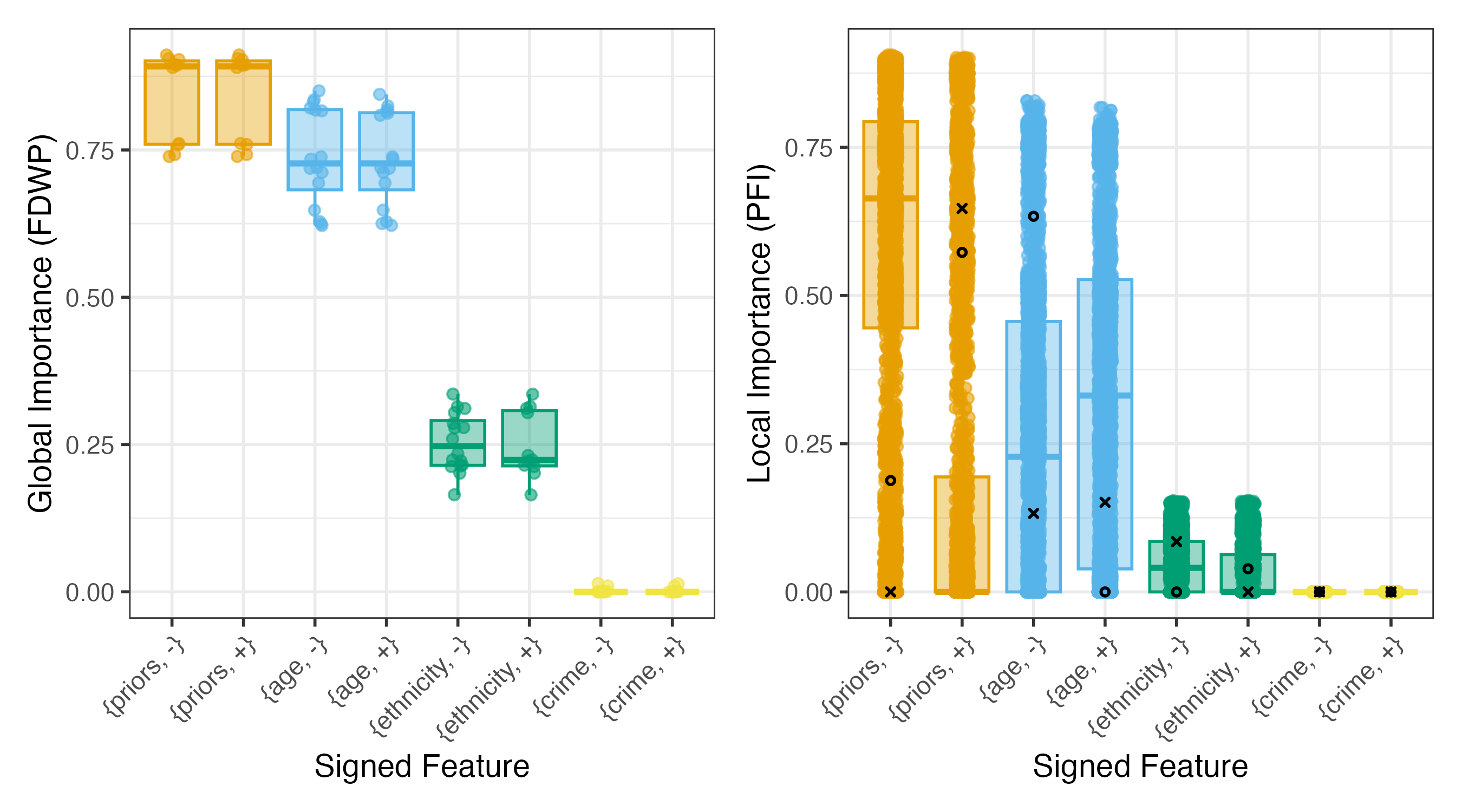}
    \caption{Same as Figure~\ref{fig:fi} above but with impurity decrease threshold $\epsilon=0.01$.}
    \label{fig:fi_sup_imp_dec_01}
\end{figure}

\begin{figure}[tb]
 \centering
    \includegraphics[width=0.5\textwidth]{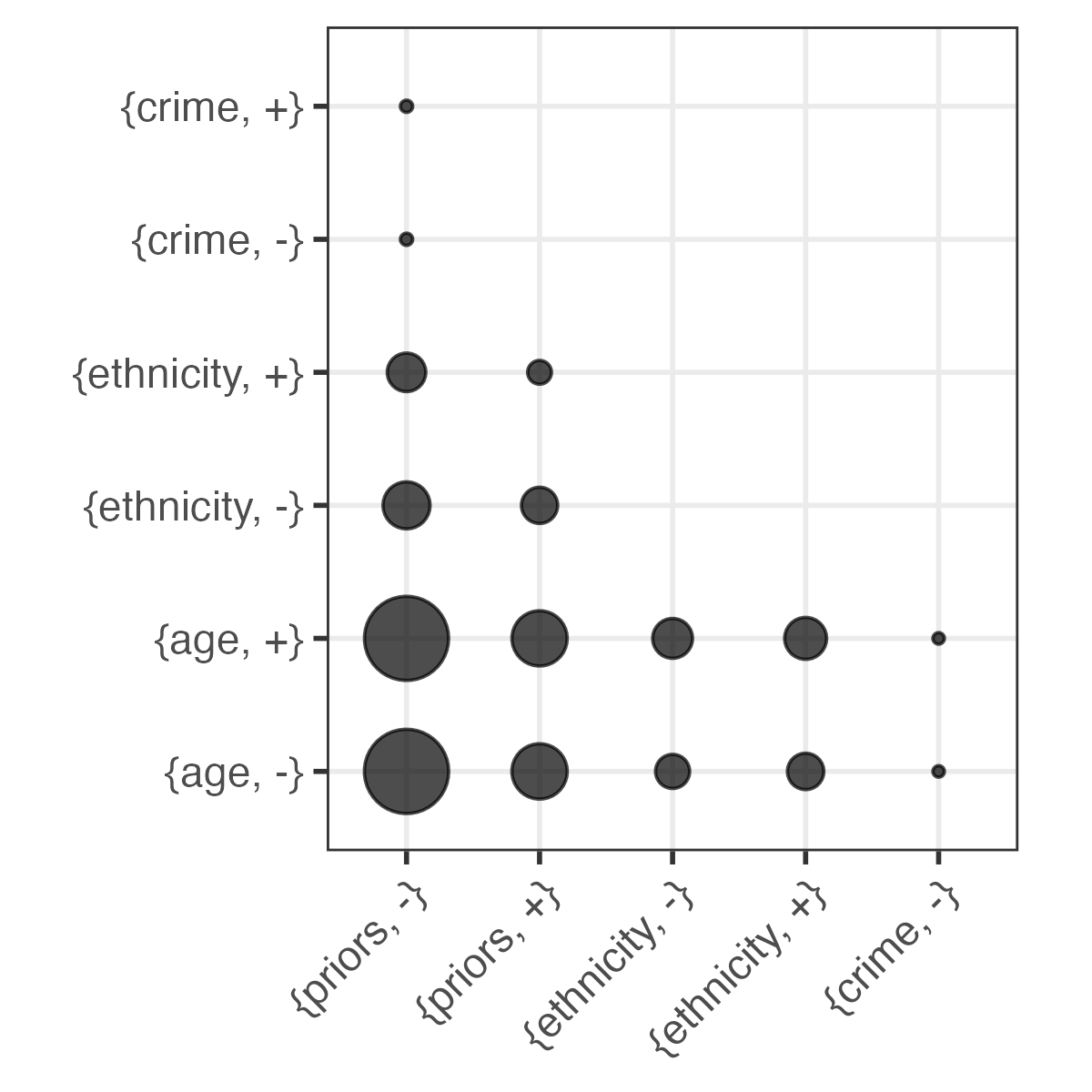}
         \caption{Same as Figure~\ref{fig:interaction} above but with impurity decrease threshold $\epsilon=0.01$.} 
         \label{fig:interaction_imp_dec_01}
\end{figure}

\begin{figure}[tb]
\centering
    \includegraphics[width=\textwidth]{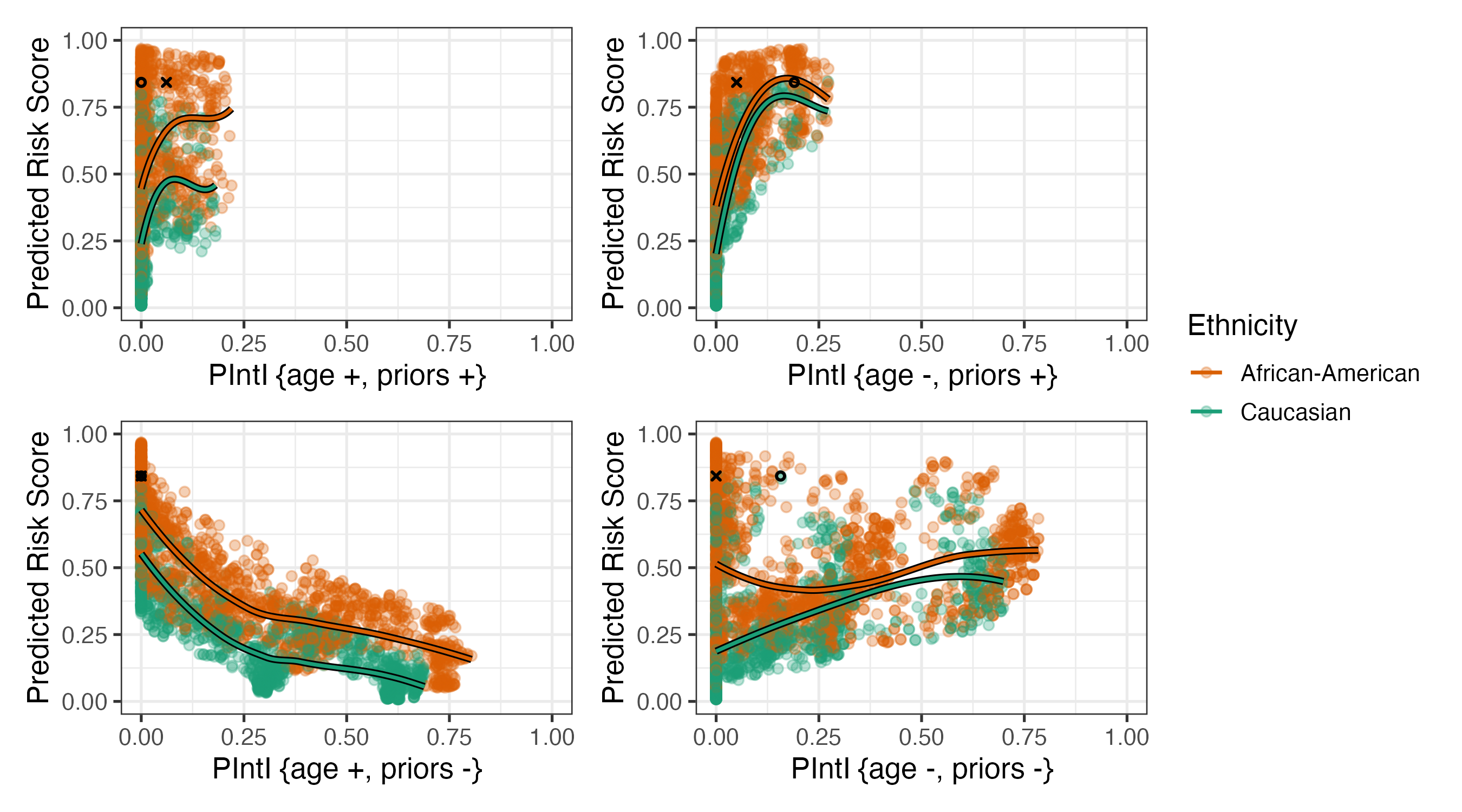}
         \caption{Same as Figure~\ref{fig:pred_age_priors} above but with impurity decrease threshold $\epsilon=0.01$.}
         \label{fig:pred_age_priors_imp_dec_01}
\end{figure}

\end{document}